%% file: main.tex
\PassOptionsToPackage{dvipsnames}{xcolor}
\documentclass{article}

\usepackage[preprint]{colm2026_conference}

\usepackage{microtype}
\usepackage{hyperref}
\usepackage{url}
\usepackage{booktabs}

\usepackage{amsmath,amssymb,amsfonts,amsthm}
\usepackage{mathtools}
\usepackage{graphicx}
\usepackage{caption}
\usepackage{multirow}
\usepackage{siunitx}
\usepackage{tabularx}
\usepackage{tcolorbox}
\tcbuselibrary{skins}
\usepackage{wrapfig}
\usepackage{xcolor}
\usepackage{colortbl}
\usepackage[capitalise,noabbrev]{cleveref}
\usepackage{enumitem}
\usetikzlibrary{arrows.meta,positioning,calc,backgrounds,fit,shapes.geometric,decorations.pathreplacing,matrix}

\setlist[itemize]{leftmargin=*,topsep=1pt,itemsep=1pt,parsep=0pt,partopsep=0pt}
\setlist[enumerate]{leftmargin=*,topsep=1pt,itemsep=1pt,parsep=0pt,partopsep=0pt}

\definecolor{Burgundy}{RGB}{92, 26, 42}
\definecolor{Gold}{RGB}{156, 122, 60}
\definecolor{GoldLight}{RGB}{201, 168, 76}
\definecolor{Navy}{RGB}{28, 50, 82}
\definecolor{Verdant}{RGB}{44, 74, 30}
\definecolor{Parchment}{RGB}{247, 243, 236}
\definecolor{Inkwell}{RGB}{26, 20, 16}
\colorlet{gray}{Inkwell!55!Parchment}
\hypersetup{colorlinks=true, citecolor=Burgundy, linkcolor=Navy, urlcolor=Verdant}
\colorlet{confRed}{Gold}
\colorlet{causalDark}{Inkwell}
\colorlet{obsBlue}{Navy}
\colorlet{expGreen}{Burgundy}
\colorlet{repOrange}{Verdant}
\colorlet{idPurple}{Burgundy}
\colorlet{estRed}{Verdant}
\colorlet{grayText}{Inkwell!70!Parchment}

\newcommand{\E}{\mathbb{E}}
\newcommand{\Pp}{\mathbb{P}}
\newcommand{\R}{\mathbb{R}}
\newcommand{\doop}{\mathrm{do}}
\newcommand{\Ind}{\mathbb{I}}
\newcommand{\clip}{\mathrm{clip}}
\newcommand{\obs}{\mathrm{OBS}}
\newcommand{\expd}{\mathrm{EXP}}
\newcommand{\simu}{\mathrm{sim}}

\colorlet{exponlycolor}{Burgundy}
\colorlet{obsonlycolor}{Navy}
\colorlet{proxyexpcolor}{Verdant}
\colorlet{groundedcolor}{Gold!85!Inkwell}
\colorlet{cvcicolor}{Burgundy!85!Inkwell}
\colorlet{cvcirescolor}{Gold!75!Burgundy}
\colorlet{rankgold}{Burgundy}
\colorlet{ranksilver}{Navy}

\newcommand{\estname}[2]{\texorpdfstring{\textcolor{#1}{#2}}{#2}}

\newcommand{\bestmean}[1]{\ifmmode{\color{rankgold}\mathbf{#1}}\else\textcolor{rankgold}{\textbf{#1}}\fi}
\newcommand{\secondmean}[1]{\ifmmode{\color{ranksilver}\mathbf{#1}}\else\textcolor{ranksilver}{\textbf{#1}}\fi}

\newcolumntype{Y}{>{\raggedright\arraybackslash}X}
\newcolumntype{H}{>{\scshape\color{Navy}}c}
\DeclareRobustCommand{\EXPOnly}{\estname{exponlycolor}{EXP-Only}}
\DeclareRobustCommand{\OBSOnly}{\estname{obsonlycolor}{OBS-Only}}
\DeclareRobustCommand{\ProxyEXP}{\estname{proxyexpcolor}{Representation-EXP}}

\DeclareRobustCommand{\Grounded}{\estname{groundedcolor}{Grounded}}

\DeclareRobustCommand{\CVCI}{\estname{cvcicolor}{CVCI}}

\DeclareRobustCommand{\CVCIRes}{\estname{cvcirescolor}{CVCI-Residual}}

\DeclareRobustCommand{\OBSStruct}{\estname{obsonlycolor}{OBS-Only}}
\DeclareRobustCommand{\RepStruct}{\estname{proxyexpcolor}{Representation-EXP}}
\DeclareRobustCommand{\GroundedStruct}{\estname{groundedcolor}{Grounded}}
\DeclareRobustCommand{\CVCIStruct}{\estname{cvcicolor}{CVCI}}
\DeclareRobustCommand{\CVCIResStruct}{\estname{cvcirescolor}{CVCI-Residual}}

\DeclareRobustCommand{\DMFull}{Direct Method with Replay (DM)}
\DeclareRobustCommand{\DR}{DR}

\DeclareRobustCommand{\DRFull}{Doubly Robust (DR)}
\newcommand{\wincell}[3]{%
  \cellcolor{#1!12}\shortstack{#2\\[-1pt]\textcolor{grayText}{\tiny +#3}}%
}

\let\oldtoprule\toprule
\let\oldmidrule\midrule
\let\oldbottomrule\bottomrule
\renewcommand{\toprule}{\arrayrulecolor{GoldLight}\oldtoprule\arrayrulecolor{Gold}}
\renewcommand{\midrule}{\arrayrulecolor{Gold}\oldmidrule}
\renewcommand{\bottomrule}{\arrayrulecolor{GoldLight}\oldbottomrule\arrayrulecolor{Gold}}
\AtBeginEnvironment{tabular}{\arrayrulecolor{Gold}}
\AtBeginEnvironment{tabularx}{\arrayrulecolor{Gold}}

\newtheorem{theorem}{Theorem}

\newtheorem{corollary}[theorem]{Corollary}
\theoremstyle{remark}

\theoremstyle{definition}

\title{The Partial Testimony of Logs: Evaluation of Language Model Generation under Confounded Model Choice}
\author{Jikai Jin and Vasilis Syrgkanis\\
\texttt{\{jkjin, vsyrgk\}@stanford.edu}}

\begin{document}

\maketitle
\lhead{}
\hypersetup{pdftitle={Causal Evaluation of Language Models under Confounded Model Choice}}

\input{sections/00_abstract}

\input{sections/01_introduction}
\input{sections/02_problem_setting}
\input{sections/03_identification}
\input{sections/04_estimators}
\input{sections/05_synthetic_experiments}
\input{sections/06_real_agent_experiments}
\input{sections/08_related_work}
\input{sections/07_discussion}
\input{sections/09_conclusion}

\bibliographystyle{colm2026_conference}
\bibliography{references}

\newpage
\appendix
\input{sections/A_additional_details}
\input{sections/B_theory_proofs_revised}

\end{document}

%% file: sections/00_abstract.tex
\begin{abstract}
Offline evaluation of language models from usage logs is biased when model choice is confounded: the same user-side factors that influence which model is used can also influence how its output is judged, so raw comparisons of logged scores mix self-selected populations rather than estimating a common quantity of interest. A small randomized experiment can break this bias by overriding model choice, but in practice such experiments are scarce and costly. We study a three-source design that combines a large confounded observational log (OBS) for scale, a small randomized experiment (EXP) for unconfounded scoring, and an offline simulator (SIM) that replays candidate models on cached contexts. Our main result is an identification theorem showing that the randomized experiment and the simulator are together enough to recover causal model values; the observational log enters only afterward, to reduce estimation error rather than to make the causal comparison valid. Six estimator families are evaluated in a controlled semi-synthetic validation and in two real-task cached benchmarks for summarization and coding. No family dominates every regime; relative performance depends on the amount of unbiased EXP supervision and on how closely the target reward aligns with OBS-derived structure.
\end{abstract}

%% file: sections/01_introduction.tex
\section{Introduction}
\label{sec:intro}

Offline evaluation of language models increasingly relies on deployment logs, preference data, and judge annotations collected in the wild \citep{Zheng2023LMSYSChat1M,Zhao2024WildChat,Zheng2023LLMJudge,Chiang2024ChatbotArena}. Once model choice is user-driven, those logs are confounded: the same unobserved factors that influence which model is used can also influence how its output is judged \citep{Kallus2018ExperimentalGrounding,Rosenman2023Shrinkage,Cheng2021Adaptive}. Raw comparisons of logged outcomes therefore mix different self-selected subpopulations rather than a common target population. In summarization, a reader who prefers concise prose may select a model known for that style and rate it more favorably; in coding assistance, repository difficulty or developer familiarity may affect both which assistant is invoked and whether the resulting patch succeeds.

A randomized experiment breaks this link by overriding model choice. In practice, such experiments are costly and disruptive, so they remain small. Prior work studies how to combine a large confounded observational sample with a small randomized one \citep{Kallus2018ExperimentalGrounding,Rosenman2023Shrinkage,Cheng2021Adaptive,Lin2025PowerLikelihood,Yang2025CVCI,Colnet2024Review}.

The problem has two parts: outcome generation and outcome scoring. Outcome generation asks which output a model would have produced on a given context; in many deployments an offline simulator (SIM) handles this by replaying models on cached contexts. Outcome scoring asks how the realized output should be evaluated without inheriting the bias in logged model choice. Separating the two yields a three-source design: a large observational log (OBS) with self-selected model choices, a small randomized experiment (EXP) with unconfounded outcome labels, and SIM outputs under alternative model choices. The goal is to estimate how each generative model would perform when logged model choice is confounded, randomized outcome labels are limited, and offline replay can regenerate candidate outputs.

\medskip
\noindent\textbf{Contributions.} The paper makes three contributions.
\begin{enumerate}[leftmargin=*,label=\textbf{(C\arabic*)}]
  \item \textbf{Identification (Theorem~\ref{thm:replay_identification}).} Under the structural assumptions stated in Section~\ref{sec:setting}, the simulator and the randomized experiment together identify causal model values. The observational log is not required for identification.
  \item \textbf{Estimator role of OBS (post-identification only).} We separate identification from estimation and compare six post-identification estimator families that differ in how they use OBS. After identification, OBS improves estimation but never serves as an identifying source. A theory-to-evidence map explains why different families help in different data regimes.
  \item \textbf{Empirical evaluation.} A controlled semi-synthetic validation and two real-task cached benchmarks (summarization, coding) compare the estimator families on held-out recommendation regret. The empirical pattern is not universal: performance depends mainly on how much unbiased experimental supervision is available, and on how well observational structure matches the reward that ultimately matters. The cached benchmarks evaluate fixed cached targets rather than the full latent population quantity one would ideally estimate.
\end{enumerate}

\medskip
Section~\ref{sec:setting} defines the causal graph, data sources, target estimands, and assumptions. Section~\ref{sec:identification} states and proves the SIM-plus-EXP identification result. Section~\ref{sec:estimators} introduces the post-identification estimator families and the theory-to-evidence map. Section~\ref{sec:benchmarks} describes the controlled validation and cached benchmarks; Sections~\ref{sec:real} and~\ref{sec:coding_public} report the cached summarization and cached coding probes. Sections~\ref{sec:related}, \ref{sec:discussion}, and~\ref{sec:conclusion} cover related work, scope and limitations, and conclusions.

%% file: sections/02_problem_setting.tex
\section{Problem setting}
\label{sec:setting}
 
OBS provides scale and auxiliary supervision from self-selected usage, EXP provides unbiased outcome labels, and SIM provides counterfactual outputs. Theorem~\ref{thm:replay_identification} shows that EXP and SIM identify the target, while OBS only improves statistical efficiency afterward.
Figures~\ref{fig:causal_graph} and~\ref{fig:three_source_pipeline} illustrate the causal structure and the division of labor.

\subsection{Variables and causal structure}

\begin{figure}[tb]
\centering
\resizebox{0.85\linewidth}{!}{%
\begin{tikzpicture}[
    transform shape,
    >=Stealth,
    every node/.style={font=\rmfamily},
    obs/.style={
      circle, draw=Navy, line width=1.2pt,
      minimum size=24pt, inner sep=0pt, fill=Parchment
    },
    latent/.style={
      circle, draw=Navy, line width=1.2pt, dashed,
      minimum size=24pt, inner sep=0pt, fill=Navy!8
    },
    causal/.style={->, line width=0.95pt, color=Burgundy},
    confound/.style={->, line width=0.9pt, color=Gold, dashed},
    annot/.style={font=\rmfamily\scriptsize\scshape, text=Inkwell, align=center},
  ]
  \path[use as bounding box] (-0.55,-1.15) rectangle (10.95,3.0);
  \node[obs]    (X) at (0,   0)   {$X$};
  \node[latent] (U) at (3.0, 2.2) {$U$};
  \node[obs]    (A) at (3.0, 0)   {$A$};
  \node[obs]    (M) at (6.0, 0)   {$M$};
  \node[obs]    (Y) at (9.0, 0)   {$Y$};
  \draw[causal] (X) -- (A);
  \draw[causal] (X) to[bend right=15] (M);
  \draw[causal] (A) -- (M);
  \draw[causal] (M) -- (Y);
  \draw[causal] (X) to[bend left=22] (Y);
  \draw[confound] (U) -- (A);
  \draw[confound] (U) to[bend left=10] (Y);
  \node[annot, below=4pt of X] {Observed\\[-1pt]context};
  \node[annot, above=4pt of U] {Latent user state\\[-1pt](unobserved)};
  \node[annot, below=4pt of A] {Chosen model};
  \node[annot, below=4pt of M] {Mediator\\[-1pt](output)};
  \node[annot, below=4pt of Y] {Outcome\\[-1pt](user utility)};
  \coordinate (leg1s) at (7.35,1.9);
  \coordinate (leg2s) at (7.35,1.5);
  \draw[confound, line width=0.9pt] (leg1s) -- ++(0.55,0);
  \draw[causal, line width=0.9pt] (leg2s) -- ++(0.55,0);
  \node[anchor=west, font=\rmfamily\scriptsize\scshape, text=Inkwell] (leg1t) at (8.05,1.9) {Confounding paths};
  \node[anchor=west, font=\rmfamily\scriptsize\scshape, text=Inkwell] (leg2t) at (8.05,1.5) {Causal paths};
  \begin{scope}[on background layer]
    \node[draw=GoldLight, fill=Parchment, rounded corners=2pt, fit=(leg1s)(leg2s)(leg1t)(leg2t), inner sep=3pt] {};
  \end{scope}
\end{tikzpicture}
}
\caption{Causal graph with confounded model choice. The latent state $U$ affects both the chosen model $A$ and the outcome $Y$, so OBS comparisons of logged outcomes are biased. Assumption~A1 (randomization in EXP) breaks the $U\to A$ link inside EXP, and Assumption~A4 (no latent mediator confounding) blocks any direct $U\to M$ path; together they enable the SIM-plus-EXP identification of Theorem~\ref{thm:replay_identification}.}
\label{fig:causal_graph}
\end{figure}
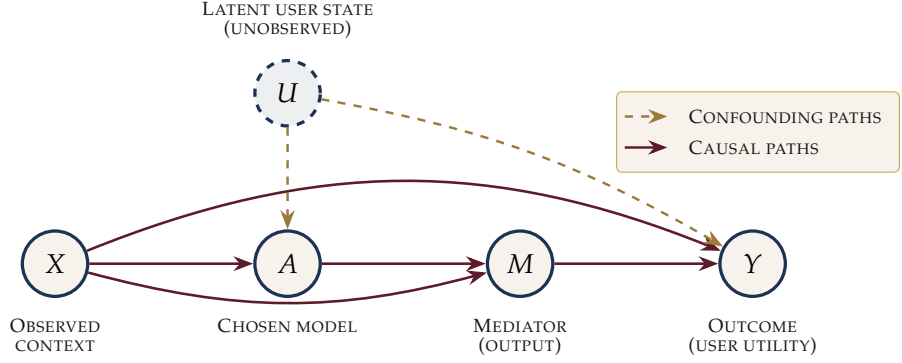

We study a five-variable graph $(X,U,A,M,Y)$ shown in Figure~\ref{fig:causal_graph}, where
$X$ is the observed context, $U\in\R^{d_u}$ is an unobserved user state,
$A\in\mathcal{A}$ indexes the chosen generative model, $M$ is the generated output,
and $Y\in[0,1]$ is the scalar outcome. Confounding arises because the latent state $U$ affects both model choice and outcome.
Beyond ordinary confounding, the graph imposes two structural restrictions: conditional on $(X,A)$, the latent user state $U$ does not directly affect the mediator $M$, and $A$ affects the outcome $Y$ only through $M$. These assumptions are natural in stateless generative systems, where the model's behavior is determined by its input and internal randomness rather than by who is asking.

\subsection{Accessible data sources}

Three data sources are available, each contributing something the others cannot provide alone.

\paragraph{Observational log (OBS).}
OBS is a large observational dataset
$D_{\obs}=\{(X_i,A_i,M_i,Y_i,Z_i)\}_{i=1}^{n_{\obs}},$
where $A_i$ may depend on the corresponding latent state $U_i$ and
$Z_i\in[0,1]^K$ are $K$ auxiliary labels that may or may not be available.
The auxiliary labels $Z_i$ provide extra supervision for learning a compact proxy representation of the context--output pair; user feedback ratings and pairwise preferences are natural sources and can often be collected at scale \citep{Christiano2017HumanPreferences,Ouyang2022InstructGPT,Chiang2024ChatbotArena}.

\paragraph{Randomized experiment (EXP).}
EXP is a much smaller dataset
$D_{\expd}=\{(X_j,A_j,M_j,Y_j)\}_{j=1}^{n_{\expd}},$
in which model choice is randomized and therefore unconfounded.
Since $A_j$ is drawn from the experimental action set $\mathcal{A}_{\expd}$ according to a fixed distribution, $A\perp U$ and $A\perp U\mid X$, so EXP outcomes are causally interpretable. Randomizing model choice disrupts user experience, which is why these experiments remain small in practice.

\paragraph{Offline simulator (SIM).}
The simulator reruns any generative model on any context and returns a mediator
draw from $M \sim p_{\simu}(\cdot\mid X=x, A=a).$
Operationally, SIM reruns generative model $a$ on held-out context $x$ and records the generated output. In generative systems, SIM is especially natural because prompts, retrieved documents, tool traces, and other environment inputs can often be stored as deterministic inputs, while the remaining randomness comes from the generative model itself.

\begin{figure}[tb]
\centering
\includegraphics[width=\textwidth]{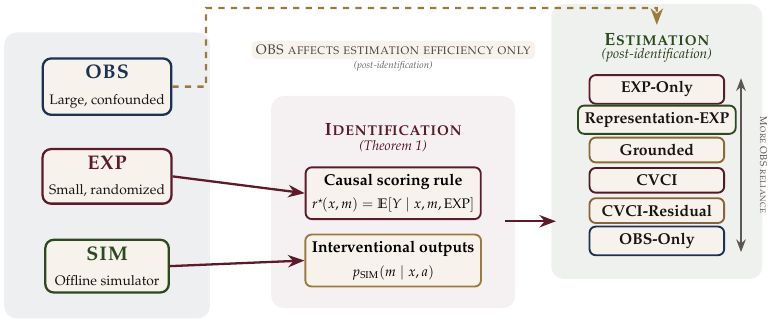}
\caption{Three-source design. SIM and EXP are sufficient for identification of $r^\star$, $q$, and $\mu$ on shared support; OBS affects estimation efficiency only. Solid arrows mark the identification flow, and dashed gold arrows mark post-identification statistical signal.}
\label{fig:three_source_pipeline}
\end{figure}

\subsection{Target estimands}

Two practical targets matter.

\paragraph{Marginal model value.}
\emph{How well would each model perform on average if used by the full target population?}
The marginal model value is the main target for offline comparison of a fixed set of generative models, and it corresponds to the target estimand
\[
\mu(a) := \E\!\left[Y\big(\doop(A=a)\big)\right].
\]

\paragraph{Context-conditional model value.}
\emph{Given a particular observed context $x$, what outcome would each model deliver on average?}
The context-conditional value is the key target for recommendation and personalization, and it corresponds to the context-level function
\[
q(x,a) := \E\!\left[Y\big(\doop(A=a)\big)\mid X=x\right].
\]

Section~\ref{sec:identification} shows how both targets are recovered: SIM supplies the model-specific mediator distribution and EXP provides unconfounded scores for realized outputs.

In real cached benchmarks, we instead evaluate the finite cached target
\[
q_{\mathrm{cache}}(x,a):=r^\star\!\bigl(x,m_{\mathrm{cache}}(x,a)\bigr),
\]
where $m_{\mathrm{cache}}(x,a)$ is the stored output for that context--model pair. The finite cached target approximates $q(x,a)$; the superpopulation estimand remains latent.

\subsection{Design assumptions for SIM-based identification}

\paragraph{Assumption A1 (Randomization in EXP).}
In EXP, $A$ is randomized over $\mathcal{A}_{\expd}$ so that $A\perp U$ and $A\perp U\mid X$ hold.

Randomization is the sole source of unconfounded outcome information; without it, identification is impossible regardless of OBS size.

\paragraph{Assumption A2 (SIM validity and support).}
The simulator matches the interventional mediator distribution, $p_{\simu}(m\mid x,a)=\Pp(M=m\mid X=x,\doop(A=a))$, and is supported within the EXP mediator support: for every $x$ in the EXP context support and every $a\in\mathcal{A}_{\expd}$,
\[
\operatorname{supp} p_{\simu}(\cdot\mid x,a)\;\subseteq\;\operatorname{supp} \Pp(M\mid X=x,A=a,\expd).
\]
Throughout the paper, ``the shared EXP/SIM support'' refers to this support condition. SIM validity is satisfied whenever any randomness in the generated output comes from the generative model itself and all deployment-time inputs that affect the generative model's behavior are included in $X$.

\paragraph{Assumption A3 (Outcome consistency under intervention).}
OBS and EXP measure the same outcome $Y$, and the conditional outcome law given $(X,M,U)$ is invariant across OBS, EXP, and the post-intervention world: for every $(x,m,u)$ on the relevant support,
\begin{align*}
\Pp(Y\mid X=x,M=m,U=u,\obs)
\;&=\;\Pp(Y\mid X=x,M=m,U=u,\expd)\\
\;&=\;\Pp(Y\mid X=x,M=m,U=u,\doop(A=a)).
\end{align*}
Only the model-choice mechanism differs across OBS, EXP, and intervention. Without this, EXP supervision could not transport either to the intervention or to the OBS-trained reward models.

\paragraph{Assumption A4 (No latent mediator confounding; single-round generation).}
Conditional on $(X,A)$, the latent state $U$ is independent of the mediator $M$ in both OBS and EXP, and under $\doop(A=a)$:
\[
U \perp M \mid X, A.
\]
Equivalently, the SCM has no edge $U\to M$ once $(X,A)$ is fixed, and any deployment-time inputs that drive the generative model are absorbed into $X$. The same structural restriction appears in Figure~\ref{fig:causal_graph}; promoting it to a formal assumption combines with A1's randomization in the mediator step of the proof of Theorem~\ref{thm:replay_identification}. A4 restricts the framework to single-round, stateless generation. In repeated, adaptive, or multi-round interactions the residual dependence of $M$ on $U$ given $(X,A)$ is generally nonzero, so A4 fails.

\paragraph{Assumption A5 (Context distribution alignment).}
The marginal model value $\mu(a)$ is defined with respect to a target context distribution $P_X^{\mathrm{tgt}}$ that coincides with the EXP context distribution: $P_X^{\mathrm{tgt}}=P_X^{\expd}$. Equivalently, the EXP sample is drawn from the same context population over which $\mu(a)$ is averaged. When this fails, only $q(x,a)$ on the shared support is identified, and any reweighting to a different $P_X^{\mathrm{tgt}}$ requires standard transportability and trial-generalizability arguments \citep{ColeStuart2010,PearlBareinboim2011Transportability,BareinboimPearl2016DataFusion} that are outside the scope of this paper.

%% file: sections/03_identification.tex
\section{Identification}
\label{sec:identification}

The remaining identification problem is to score generated outputs without bias. EXP provides that scoring rule, while SIM provides the interventional mediator distribution.
Throughout, we evaluate only actions in the experimental action set $\mathcal{A}_{\expd}$ and mediator
values on the shared support of the EXP law of $M\mid X=x,A=a$ and the corresponding
SIM law $p_{\simu}(\cdot\mid x,a)$. Define \(r^\star(x,m) := \E[Y \mid X=x, M=m, \expd]\).

\begin{theorem}[SIM plus EXP identifies causal value]
\label{thm:replay_identification}
Under the causal graph in Figure~\ref{fig:causal_graph} and Assumptions A1--A5, for every $x$
in the EXP context support and every $a\in \mathcal{A}_{\expd}$,
\begin{equation}
q(x,a)
=
\E\!\left[Y\big(\doop(A=a)\big)\mid X=x\right]
=
\E_{M\sim p_{\simu}(\cdot\mid x,a)}\!\left[r^\star(x,M)\right],
\label{eq:ident_q}
\end{equation}
and therefore, with $\mu(a)$ defined under the target context distribution $P_X^{\mathrm{tgt}}=P_X^{\expd}$ of Assumption A5,
\begin{equation}
\mu(a)
=
\E_{X\sim P_X^{\mathrm{tgt}}}[q(X,a)]
=
\E_{X\sim P_X^{\mathrm{tgt}}}\!\left[\E_{M\sim p_{\simu}(\cdot\mid X,a)}\!\left[r^\star(X,M)\right]\right].
\label{eq:ident_mu}
\end{equation}
Hence both $q(x,a)$ and $\mu(a)$ are identified from EXP and SIM on the shared EXP/SIM support. In particular,
OBS need not satisfy any unconfoundedness condition for identification.
\end{theorem}

\paragraph{Proof sketch.}
The proof uses the assumptions one at a time in four steps.
\begin{enumerate}[leftmargin=*,topsep=2pt,itemsep=1pt,label=\textbf{Step \arabic*.}]
\item \emph{Combine A1 and A4 to obtain conditional independence.} Assumption~A1 gives $A\perp U\mid X$ in EXP, and Assumption~A4 gives $U\perp M\mid X,A$ in EXP and under intervention. The chain rule yields $U\perp(A,M)\mid X$ in EXP, so the EXP conditional regression $r^\star(x,m)=\E[Y\mid X=x,M=m,\expd]$ is unbiased for the causal scoring rule on the shared EXP/SIM support.
\item \emph{Use A3 to transport scoring across OBS, EXP, and intervention.} Assumption~A3 states that the conditional outcome law given $(X,M,U)$ is the same in OBS, in EXP, and under $\doop(A=a)$. Combined with Step~1, this means $r^\star(x,m)$ is also the conditional outcome regression under $\doop(A=a)$ on the shared support.
\item \emph{Use A2 to plug in the SIM mediator distribution.} Assumption~A2 states that $p_{\simu}(m\mid x,a)$ is the interventional mediator distribution. Marginalizing $r^\star$ over $M$ under $p_{\simu}(\cdot\mid x,a)$ therefore identifies $q(x,a)$, giving \eqref{eq:ident_q}.
\item \emph{Use A5 to identify the marginal model value.} Assumption~A5 fixes the target context distribution to $P_X^{\mathrm{tgt}}=P_X^{\expd}$, so averaging $q(X,a)$ under that distribution identifies $\mu(a)$, giving \eqref{eq:ident_mu}.
\end{enumerate}
The full proof is in Appendix~\ref{app:proof_replay_identification}. Identification requires both SIM, to regenerate outputs, and EXP, to provide outcome labels for those outputs. Conditioning on $(X,M)$ in OBS does not block confounding through $U$ when Assumption~A4 fails or OBS routes $A$ on $U$. Appendix~\ref{app:id_remarks} gives the explicit derivations.

\paragraph{From identification to the experimental design.}
The controlled semi-synthetic validation and the real cached benchmarks probe different parts of the theorem. The controlled semi-synthetic validation (Section~\ref{sec:controlled_bridge}) directly probes the theorem-level target: a known latent reward generator varies only $\beta$, $n_{\obs}$, and $n_{\expd}$, so estimator gaps reflect SIM-plus-EXP identification followed by post-identification estimation. The real cached benchmarks (Sections~\ref{sec:real} and~\ref{sec:coding_public}) report recommendation regret against the finite cached target $q_{\mathrm{cache}}$, with one cached output per context-action pair and synthetic OBS/EXP resampling. They diagnose the post-identification estimator families within those benchmark constructions rather than testing Theorem~\ref{thm:replay_identification} directly. In all settings, the estimator design in Section~\ref{sec:estimators} uses OBS only after identification to reduce variance or improve approximation.

%% file: sections/04_estimators.tex
\section{Estimators}
\label{sec:estimators}

After SIM and EXP identify the target, the remaining question is how OBS can reduce estimation error. The six estimator families differ only in how they use OBS after identification: some ignore OBS outcomes entirely, some use OBS only to learn a lower-dimensional proxy representation, and some use OBS outcomes directly through a baseline or pooled fit. Exact objectives, proxy-learning details, and value-aggregation formulas are deferred to Appendix~\ref{app:estimator_details} and Appendix~\ref{app:value_estimators}.

Throughout, $z=(x,m)$ denotes a context--mediator pair, and all reward models start from text features $\varphi(z)$. Hybrid methods use OBS auxiliary labels only to learn a lower-dimensional proxy representation $\psi(z)$, in the same spirit as text embeddings adapted for causal adjustment \citep{VeitchSridharBlei2020}. The proxy only makes the EXP-side fit lower-dimensional and leaves the identified target unchanged.

\begin{table}[tb]
\centering
\scriptsize
\setlength{\tabcolsep}{3pt}
\renewcommand{\arraystretch}{0.95}
\caption{How the six reward estimators use OBS and EXP. Darker shading in the OBS-role column indicates heavier reliance on observational signal. Section~\ref{sec:estimators} gives the prediction form for each family.}
\label{tab:reward_estimators}
\hfuzz=8pt
\begin{tabular}{@{}>{\raggedright\arraybackslash}p{1.6cm}>{\raggedright\arraybackslash}p{3.9cm}>{\raggedright\arraybackslash}p{3.8cm}>{\raggedright\arraybackslash}p{3.8cm}@{}}
\toprule
Estimator & OBS role & EXP role & Bias--variance intuition \\
\midrule
\EXPOnly{} & None & Fits the reward model & No confounding bias, but high variance when EXP is small. \\
\OBSOnly{} & \cellcolor{Navy!24!Parchment}Fits the reward model on logged outcomes & None & Low variance, but inherits bias from confounded model choice. \\
\begin{tabular}[t]{@{}l@{}}\textcolor{proxyexpcolor}{Representation-}\\\textcolor{proxyexpcolor}{EXP}\end{tabular} & \cellcolor{Navy!10!Parchment}Learns the proxy representation $\psi$ from auxiliary labels & Fits the reward head on $\psi$ & Helps when the target varies along directions retained by $\psi$; hurts when $\psi$ omits target-relevant directions. \\
\Grounded{} & \cellcolor{Navy!16!Parchment}Provides the baseline predictor $f_{\obs}$ and the proxy space $\psi$ & Estimates and tunes a correction to $f_{\obs}$ & Uses OBS for efficiency and EXP to remove low-dimensional bias. \\
\CVCI{} & \cellcolor{Navy!30!Parchment}Contributes logged-outcome loss directly & Chooses the pooling weight by causal cross-validation and anchors the fit to randomized outcomes & Trades OBS bias against finite-EXP estimation error through direct pooling. \\
\begin{tabular}[t]{@{}l@{}}\textcolor{cvcirescolor}{CVCI-}\\\textcolor{cvcirescolor}{Residual}\end{tabular} & \cellcolor{Navy!20!Parchment}Provides the baseline and contributes pooled residual loss in $\psi$ & Chooses the pooling and shrinkage of the residual correction & Uses OBS for the coarse fit and EXP for the residual mismatch. \\
\bottomrule
\end{tabular}
\hfuzz=0.1pt
\end{table}

The estimator comparison is guided by three evidence tiers. The identification theorem is part of the main paper. The estimator-comparison results in Appendix~\ref{app:theory} provide formal support for several oracle and approximation statements. Target-alignment patterns in the benchmarks diagnose whether a target reward aligns with a proxy, baseline, or correction class. They provide empirical rather than formal evidence. Table~\ref{tab:theory_evidence_bridge} summarizes this map.

\begin{table}[tb]
\centering
\scriptsize
\setlength{\tabcolsep}{3pt}
\renewcommand{\arraystretch}{1.08}
\caption{Theory-to-evidence map for the estimator comparison. Each row names a mechanism, its evidence tier, the formal support when present, the empirical diagnostic that probes it, and the rule for interpreting wins and losses on that diagnostic. Tier 1 is formal support in the main paper; Tier 2 is formal support from the appendix; Tier 3 is empirical diagnostic evidence only.}
\label{tab:theory_evidence_bridge}
\begin{tabularx}{\textwidth}{@{}p{2.3cm}p{1.15cm}p{3.2cm}p{4.0cm}X@{}}
\toprule
Mechanism & Tier & Formal support & Empirical diagnostic & Interpretation rule \\
\midrule
SIM plus EXP identify the target & 1 & Theorem~\ref{thm:replay_identification} & All benchmarks separate randomized scoring from replayed or cached outputs & EXP supplies unconfounded scores; SIM or cached replay supplies mediators. \\
\Grounded{} can enlarge a proxy-side approximation class & 2 & Theorem~\ref{thm:grounded_function_class} & Basis-expanded \Grounded{} improves over the single-linear grounded baseline in Appendix~\ref{app:grounded_real_upgrade} & Treat this as benchmark evidence for richer correction bases; the theorem establishes class expansion, not finite-sample dominance. \\
Oracle correction versus finite-sample correction & 2 & Theorem~\ref{thm:grounded_vs_obs} and \eqref{eq:noisy_correction_identity} & \Grounded{} is competitive but not uniformly best in the cached benchmarks & The oracle result is not a finite-sample dominance guarantee. \\
Residual simplicity & 2 & Theorem~\ref{thm:residual_vs_pooling} & \CVCIRes{} wins only selected cells and often underperforms in real summarization & Read wins or losses as benchmark evidence about residual simplicity, not as direct tests of the oracle condition. \\
Proxy alignment and reward definition & 3 & Empirical diagnostic unless additional approximation-class results are formalized & Summarization \(R^2\) diagnostics and coding fix-success sweep & Interpret as empirical diagnostics rather than formal oracle statements. \\
\bottomrule
\end{tabularx}
\end{table}

\subsection{Reward-model estimators}

All six reward models minimize penalized squared loss with predictions clipped to $[0,1]$. They differ only in whether and how they use OBS. Table~\ref{tab:reward_estimators} summarizes the six families; Appendix~\ref{app:estimator_details} gives exact objectives and benchmark-specific instantiations.

\paragraph{Prediction-form summary.}
Each family produces a clipped predicted reward $\widehat r(x,m)$ from one of the following forms; here $\varphi(z)$ is the raw text feature, $\psi(z)$ is the OBS-derived proxy representation, $f_{\obs}(z)$ is the OBS-trained predictor (so $\widehat r_{\obs}(z)=f_{\obs}(z)$ for $\OBSOnly{}$), and $\alpha_{\mathrm{corr}}\in[0,1]$, $\lambda\in[0,1]$ are EXP-side tuning weights chosen by causal cross-validation:
\begin{align*}
\widehat r_{\expd}(z) &= \clip_{[0,1]}\!\bigl(\widehat w_{\expd}^{\top}\varphi(z)+\widehat b_{\expd}\bigr), \\
\widehat r_{\psi}(z) &= \clip_{[0,1]}\!\bigl(\widehat w_{\psi}^{\top}\psi(z)+\widehat b_{\psi}\bigr), \\
\widehat r_{\mathrm{G}}(z) &= \clip_{[0,1]}\!\Bigl(f_{\obs}(z)-\alpha_{\mathrm{corr}}\bigl(\widehat\theta_B^{\top}B(\psi(z))+\widehat c_B\bigr)\Bigr), \\
\widehat r_{\mathrm{C}}(z) &= \clip_{[0,1]}\!\bigl(\widehat w_\lambda^{\top}\varphi(z)+\widehat b_\lambda\bigr), \\
\widehat r_{\mathrm{CR}}(z) &= \clip_{[0,1]}\!\bigl(f_{\obs}(z)+\widehat g_{\mathrm{res}}(z)\bigr).
\end{align*}
In words: $\widehat r_{\expd}$ regresses $Y$ on $\varphi$ using EXP; $\widehat r_{\psi}$ regresses $Y$ on the OBS-derived proxy $\psi$ using EXP; and $\widehat r_{\mathrm{G}}$ adds an EXP-fit correction to the OBS baseline $f_{\obs}$ in proxy space. $\widehat r_{\mathrm{C}}$ pools the OBS and EXP regression problems on $\varphi$ with mixing weight $\lambda$ ($\lambda=1$ gives $\OBSOnly{}$, $\lambda=0$ gives $\EXPOnly{}$), and $\widehat r_{\mathrm{CR}}$ adds an EXP residual fit $\widehat g_{\mathrm{res}}$ in proxy space to $f_{\obs}$ on the residual target $Y-f_{\obs}(z)$. Appendix~\ref{app:estimator_details} gives the exact objectives and the tuning rule for $\alpha_{\mathrm{corr}}$, $\lambda$, and the basis $B$.

\paragraph{From reward models to model values.}
Given a fitted reward $\widehat r$, the direct method (DM) and doubly robust (DR) value estimators \citep{RobinsRotnitzkyZhao1994,BangRobins2005,Dudik2011DoublyRobust,Dudik2014DRStatSci,Jiang2016DROffPolicy} plug $\widehat r$ into the identification formula \eqref{eq:ident_q}--\eqref{eq:ident_mu}. For each held-out context $x$ and action $a$, the cached benchmarks compute
\[
\widehat q^{\mathrm{DM}}(x,a)=\widehat r(x,m_{\mathrm{cache}}(x,a)),
\]
i.e.\ the fitted reward at the cached mediator; the controlled semi-synthetic validation, which has access to a real SIM, instead averages $\widehat r$ over a SIM mediator sample. The marginal value is $\widehat\mu^{\mathrm{DM}}(a)=\E_{X\sim P_X^{\expd}}[\widehat q^{\mathrm{DM}}(X,a)]$, with the expectation replaced by an empirical average over the held-out EXP context split.

The doubly robust estimator (DR) corrects $\widehat\mu^{\mathrm{DM}}(a)$ on EXP. With $K$-fold cross-fitting and uniform EXP randomization $p_{\expd}(a)=1/|\mathcal A_{\expd}|$, write $\widehat r^{(-k)}$ for the reward model trained on EXP folds $\{1,\ldots,K\}\setminus\{k\}$ and let $k(j)$ be the fold containing EXP point $j$. The DR estimator for the marginal model value is
\begin{equation}
\widehat\mu^{\mathrm{DR}}(a)
\;=\;
\widehat\mu^{\mathrm{DM}}(a)
\;+\;
\frac{1}{n_{\expd}}\sum_{j=1}^{n_{\expd}}
\frac{\Ind\{A_j=a\}}{p_{\expd}(a)}\,\widehat\varepsilon_j(a),
\label{eq:dr-mu}
\end{equation}
where $\widehat\varepsilon_j(a):=Y_j-\widehat r^{(-k(j))}(X_j,M_j)$ is the cross-fitted EXP residual. Both terms in \eqref{eq:dr-mu} target the marginal value $\mu(a)$ under Assumption~A5. The complete DM/DR formulas and cross-fitting protocol are given in Appendix~\ref{app:value_estimators}.

\begin{enumerate}[leftmargin=*]
\item \textbf{Use only EXP.} EXP-Only is the unbiased baseline: it fits a clipped ridge reward model using only randomized EXP outcomes (the standard randomized-trial regression baseline used as a reference point in hybrid OBS--EXP work, e.g.\ \citealp{Kallus2018ExperimentalGrounding,Rosenman2023Shrinkage,Yang2025CVCI}), and pays for that with high variance when EXP is small.

\item \textbf{Use OBS only to learn a representation.} Representation-EXP still relies on EXP for causal supervision, but it first uses OBS auxiliary labels to learn $\psi(z)$ and then fits the reward head on that representation \citep{VeitchSridharBlei2020}. It helps when the target varies mainly along directions preserved by $\psi$, and it hurts when $\psi$ discards directions that matter for the target.

\item \textbf{Use OBS outcomes, calibrated by EXP.} These methods use OBS for signal and EXP for calibration. $\OBSOnly{}$ is the uncalibrated endpoint --- a clipped ridge regression on OBS outcomes alone, the canonical biased-but-low-variance baseline in the OBS--EXP combination literature \citep{Kallus2018ExperimentalGrounding,Rosenman2023Shrinkage,Cheng2021Adaptive,Lin2025PowerLikelihood,Yang2025CVCI} --- and the next three methods differ only in how they calibrate it.
\end{enumerate}

\textbf{3.1 Grounded Correction (\Grounded{}).}
Inspired by experimental grounding \citep{Kallus2018ExperimentalGrounding}, \Grounded{} starts from the OBS predictor $f_{\obs}$ and uses EXP to learn a correction in proxy space. In the benchmark implementation, that correction is chosen from a small basis library rather than from a single linear head. Its final prediction is
\[
r_{\mathrm{ground},B,\alpha_{\mathrm{corr}}}(z)
\;=\;
\clip_{[0,1]}\!\Big(f_{\obs}(z)-\alpha_{\mathrm{corr}}\big(\widehat\theta_{B}^{\top}B(\tilde\psi(z))+\widehat c_{B}\big)\Big),
\]
where $\tilde\psi(z)$ is the low-dimensional proxy representation and $B$ is selected from a small fixed basis library on that proxy space. The method is most attractive when OBS already captures the coarse reward surface and EXP mainly needs to remove a lower-dimensional bias. Appendix~\ref{app:estimator_details} gives the exact fitting and tuning details.

\textbf{3.2 Cross-validated Causal Inference (\CVCI{}).}
CVCI \citep{Yang2025CVCI}, building on the broader OBS--EXP combination literature \citep{Rosenman2023Shrinkage,Cheng2021Adaptive,Lin2025PowerLikelihood}, pools the OBS and EXP regression problems directly and chooses the pooling weight $\lambda\in[0,1]$ by causal cross-validation on EXP. Concretely, the pooled coefficient solves
\begin{equation}
(\widehat w_\lambda,\widehat b_\lambda)
=
\arg\min_{w,b}\Bigl\{
(1-\lambda)\,\mathcal L_{\expd}(w,b;\varphi,Y)
+\lambda\,\mathcal L_{\obs}(w,b;\varphi,Y)
+\alpha\|w\|_2^{2}
\Bigr\},
\label{eq:cvci-objective}
\end{equation}
where $\mathcal L_{\mathcal D}(w,b;\phi,Y):=\sum_{i\in\mathcal D}(Y_i-(w^{\top}\phi(z_i)+b))^{2}$ is the squared-error loss on dataset $\mathcal D$ in feature map $\phi$,
so $\lambda=1$ recovers the $\OBSOnly{}$ regression and $\lambda=0$ recovers the $\EXPOnly{}$ regression on $\varphi$. The pooling weight $\lambda$ is selected by minimizing held-out causal-CV loss on EXP at the model level (Appendix~\ref{app:agent_cv}), anchoring the bias--variance interpolation to randomized outcomes.

\textbf{3.3 CVCI-Residual.}
CVCI-Residual extends the CVCI pooling rule of \citet{Yang2025CVCI} by first residualizing around the OBS baseline (replacing the regression target $Y$ with $Y-f_{\obs}(z)$) and then running a $\lambda$-pooled fit in the proxy space $\psi(z)$:
\hfuzz=5pt
\begin{equation}
(\widehat w^{\mathrm{res}}_\lambda,\widehat b^{\mathrm{res}}_\lambda)
=
\arg\min_{w,b}\Bigl\{
(1-\lambda)\,\mathcal L_{\expd}(w,b;\psi,Y-f_{\obs})
+\lambda\,\mathcal L_{\obs}(w,b;\psi,Y-f_{\obs})
+\alpha\|w\|_2^{2}
\Bigr\},
\label{eq:cvci-res-objective}
\end{equation}
yielding $\widehat g_{\mathrm{res}}(z)=(\widehat w^{\mathrm{res}}_\lambda)^{\top}\psi(z)+\widehat b^{\mathrm{res}}_\lambda$ and the final prediction $\clip_{[0,1]}\!\big(f_{\obs}(z)+\widehat g_{\mathrm{res}}(z)\big)$. It targets settings where OBS already captures the hard part of the reward surface and EXP is mainly needed to calibrate a simpler discrepancy. When residual-CVCI hyperparameter CV scores tie within numerical tolerance, the tie is broken toward the more regularized candidate; Appendix~\ref{app:estimator_details} and Appendix~\ref{app:agent_cv} give the full tuning protocol and additional implementation details.

\paragraph{Family-level interpretation.}
The six items above are estimator families with benchmark-specific instantiations. The proxy channel, the OBS baseline, the basis library, the pooling rule, and the EXP-side tuning protocol differ across benchmarks (see Table~\ref{tab:concrete_impls} in Appendix~\ref{app:estimator_details}). The empirical comparisons in Sections~\ref{sec:benchmarks}--\ref{sec:coding_public} diagnose those family-level instantiations.

%% file: sections/05_synthetic_experiments.tex
\section{Empirical benchmarks}
\label{sec:benchmarks}

\subsection{Benchmark setup}

Across all empirical benchmarks, the task is the same: estimate a reward surface, recommend the best action on each held-out context, and measure held-out recommendation regret.
Two quantities recur throughout the benchmarks: $\beta$ controls the OBS choice law, and the real cached benchmarks replace the latent target $q$ with the finite cached target $q_{\mathrm{cache}}$.
Each benchmark setting specifies a training-context distribution, a held-out test set, an action set, a simulator, an experimental reward map, and an observational logging family.
Within one such cell, the observational law factors as
\[
p_{\obs}(x,u,a,m,y)
=
P_{\mathrm{tr}}(x)\,P(u\mid x)\,p_{\obs}^{\beta}(a\mid x,u)\,p_{\simu}(m\mid x,a)\,P_Y(y\mid x,m,u),
\]
and the experimental law replaces the choice model by an unconfounded randomization rule \(\pi_{\expd}(a\mid x)\):
\[
p_{\expd}(x,u,a,m,y)
=
P_{\mathrm{tr}}(x)\,P(u\mid x)\,\pi_{\expd}(a\mid x)\,p_{\simu}(m\mid x,a)\,P_Y(y\mid x,m,u).
\]
Here \(P_Y\) is the benchmark-specific outcome law. In the main-text benchmark settings, \(\pi_{\expd}(a\mid x)=1/|\mathcal A|\), and \(r^\star(x,m)=\E[Y\mid X=x,M=m,\expd]\) integrates out \(U\mid X=x\) under \(p_{\expd}\).
The identified target remains
\[
q(x,a)=\E_{M\sim p_{\simu}(\cdot\mid x,a)}[r^\star(x,M)].
\]
Each estimator returns \(\widehat q\), inducing the recommendation rule \(\widehat\pi(x)=\arg\max_{a\in\mathcal A}\widehat q(x,a)\).
We evaluate that rule by held-out recommendation regret,
\[
\mathrm{Regret}_{\mathrm{test}}
=
\frac{1}{|\mathcal{X}_{\mathrm{test}}|}
\sum_{x\in\mathcal{X}_{\mathrm{test}}}
\Big[
\max_{a\in\mathcal{A}} q(x,a)-q(x,\widehat\pi(x))
\Big].
\]

The \(\beta\)-indexed observational choice law is benchmark-specific. In the appendix controlled validation, \(p_{\obs}^{\beta}(a\mid x,u)=(1-\beta)\pi_X(a\mid x)+\beta\,\pi_U(a\mid u)\), where \(\pi_X\) is a context-driven policy and \(\pi_U\) is a latent-user policy. Thus \(\beta=0\) removes latent-user routing, whereas \(\beta=1\) uses only the latent-user component. Section~\ref{sec:real} instead uses a softmax router in which \(\beta\) scales the latent-user logit contribution.

The real cached benchmarks fix one realized mediator \(m_{\mathrm{cache}}(x,a)\) for each context-action pair and report the finite cached target \(q_{\mathrm{cache}}(x,a)=r^\star\!\bigl(x,m_{\mathrm{cache}}(x,a)\bigr)\). Their reported recommendation regret therefore replaces \(q\) by \(q_{\mathrm{cache}}\) on the held-out cache. Randomness then comes from OBS/EXP sampling and the benchmark's latent draws, not from regenerating mediators at evaluation time.

\subsection{Controlled validation as a theory bridge}
\label{sec:controlled_bridge}

The controlled semi-synthetic validation uses the summarization task family but replaces judged rewards by a known latent reward generator. The validation keeps the SIM/EXP/OBS structure fixed while varying confounding strength \(\beta\), the observational budget \(n_{\obs}\), and the randomized budget \(n_{\expd}\). The validation connects the estimator theory to observable behavior. When EXP is scarce, OBS-assisted estimators can improve held-out recommendation performance. With a larger randomized sample, EXP-only and directly pooled estimators become competitive. No single estimator dominates every regime. The full generator, protocol, and \(24\)-cell winner map are in Appendix~\ref{app:synth_validation} and Appendix~\ref{app:synth_regime_details}.

This controlled validation isolates how data budget and confounding strength change performance when the true interventional surface is known. The real summarization and coding benchmarks below then test the same estimator families on finite cached targets \(q_{\mathrm{cache}}\), where reward definition and auxiliary-feature alignment become the main explanatory variables.

\subsection{Main-text benchmarks}

The main paper reports the controlled validation and two real-task cached benchmarks with synthetic OBS/EXP resampling. The first cached benchmark is summarization on CNN/DailyMail\footnote{\url{https://huggingface.co/datasets/ccdv/cnn_dailymail}} \citep{Hermann2015CNN,Nallapati2016Sum,See2017Pointer}. The second is coding on SWE-bench Verified\footnote{\url{https://huggingface.co/datasets/princeton-nlp/SWE-bench_Verified}}, a human-validated 500-instance subset of SWE-bench \citep{Jimenez2024SWEBench,OpenAI2024SWEBenchVerified}, with candidate patches drawn from the BouncerBench public patch pool\footnote{\url{https://github.com/uw-swag/BouncerBench/releases/download/paper/all_patches.csv}} \citep{Mathews2025BouncerBench}. Both benchmarks use real tasks, real candidate model outputs, and real LLM-judged or program-test rewards on those outputs. Their OBS and EXP mechanisms are benchmark constructions obtained by resampling and reweighting the same cached pool under specified choice and randomization rules. Table~\ref{tab:benchmark_snapshot} lists the source data, cached objects, and targets. All benchmarks use family-level instantiations of the estimators in Section~\ref{sec:estimators}, and the reported results diagnose those families within the benchmark constructions.

\begin{table}[tb]
\centering
\scriptsize
\setlength{\tabcolsep}{4pt}
\renewcommand{\arraystretch}{1.12}
\caption{Benchmarks used in the main text. The two cached benchmarks evaluate finite cached targets $q_{\mathrm{cache}}$ rather than the latent superpopulation target $q$; their OBS and EXP samples are constructed by resampling and reweighting the cached pool under specified routing and randomization rules. Each row states what the benchmark evaluates and which target it uses.}
\label{tab:benchmark_snapshot}
\begin{tabularx}{\textwidth}{@{}p{2.2cm}p{4.1cm}p{3.1cm}X@{}}
\toprule
Benchmark & Source data & Cached object & Evaluation target \\
\midrule
Controlled semi-synthetic & CNN/DailyMail article family with a known latent reward generator & SIM-generated summaries from the generator pool & True interventional reward surface $q$ (theorem-level target). \\
Real-task cached summarization & 48 difficult CNN/DailyMail articles; 20 candidate models & One judged summary per article--model pair & Finite cached target $q_{\mathrm{cache}}^{\mathrm{sm}}$ (or $q_{\mathrm{cache}}^{\mathrm{sh}}$): segment-average rubric score over four user types under the smooth (or sharpened) reward map. \\
Real-task cached coding & SWE-bench Verified issues with a BouncerBench patch pool & One cached candidate patch per issue--agent pair & Finite cached target $q_{\mathrm{cache}}$: user-average utility over fix success ($c_1$) and patch-quality components ($c_2,c_3,c_4$). \\
\bottomrule
\end{tabularx}
\end{table}

%% file: sections/06_real_agent_experiments.tex
\subsubsection{Real-task cached summarization benchmark with synthetic OBS/EXP resampling}
\label{sec:real}

The cached summarization benchmark uses CNN/DailyMail articles and LLM-judged rubric scores. Its OBS and EXP samples are constructed by resampling and reweighting the same cached pool under specified routing and randomization rules. The summarization benchmark fixes one judged cached summary for each article--model pair and evaluates two reward maps on that same cache (Table~\ref{tab:benchmark_snapshot}). Each cached summary carries a rubric vector \(s(x,a)=\bigl(s_{\mathrm{faith}}(x,a),\,s_{\mathrm{cov}}(x,a),\,s_{\mathrm{clar}}(x,a),\,s_{\mathrm{conc}}(x,a)\bigr)\in[0,1]^4\) for faithfulness, coverage, clarity, and conciseness, together with an unsupported-claims count \(u(x,a)\).
Let \(\mathcal T\) denote the four user segments and let \(w_\tau\in\Delta^4\) be the segment-specific weight vector over these rubric dimensions. For a sharpening exponent \(\gamma\ge 1\) and unsupported-claims penalty \(\lambda\ge 0\), define \(\bar w_{\tau k}^{(\gamma)}=\frac{w_{\tau k}^{\gamma}}{\sum_{\ell=1}^4 w_{\tau \ell}^{\gamma}}\) and \(\ell_\tau^{(\gamma,\lambda)}(x,a)=\sum_{k=1}^4 \bar w_{\tau k}^{(\gamma)} s_k(x,a)-\lambda \frac{u(x,a)}{20}\), and write \([t]_{[0,1]}:=\min\{1,\max\{0,t\}\}\). The smooth map used in Tables~\ref{tab:front_runner_summary} and~\ref{tab:summarization_budget_perf} is
\[
g_\tau^{\mathrm{sm}}(x,a)=\ell_\tau^{(1,0)}(x,a)=\langle w_\tau, s(x,a)\rangle,
\qquad
q_{\mathrm{cache}}^{\mathrm{sm}}(x,a)=\frac{1}{|\mathcal T|}\sum_{\tau\in\mathcal T} g_\tau^{\mathrm{sm}}(x,a).
\]
The sharpened map used in Table~\ref{tab:summarization_geometry_perf}, with \(\sigma(t)=(1+e^{-t})^{-1}\), is
\[
g_\tau^{\mathrm{sh}}(x,a)
=
\sigma\!\left(
\frac{[\ell_\tau^{(16,0.1)}(x,a)]_{[0,1]}-0.9}{0.02}
\right),
\qquad
q_{\mathrm{cache}}^{\mathrm{sh}}(x,a)=\frac{1}{|\mathcal T|}\sum_{\tau\in\mathcal T} g_\tau^{\mathrm{sh}}(x,a).
\]

OBS and EXP are sampled only from training contexts. User segments are drawn uniformly from \(\mathcal T\), EXP randomizes uniformly over \(\mathcal A\), and OBS uses the \(\beta\)-indexed router defined by setting \(\zeta_a(x,\tau)=\eta_a(x)+\beta\,\xi_a(\tau)\) and writing
\begin{equation*}
p_{\obs}^{\beta}(a\mid x,\tau)
=
(1-\varepsilon)\,
\frac{\exp(\zeta_a(x,\tau))}
{\sum_{a'\in\mathcal A}\exp(\zeta_{a'}(x,\tau))}
+\frac{\varepsilon}{|\mathcal A|}.
\end{equation*}

where \(f_{\mathrm{ctx}}(x)\) is a fixed article feature vector, \(e(\tau)\in\{0,1\}^{|\mathcal T|}\) is the one-hot segment code, and \(c_a\) is the normalized cost of model \(a\). The term \(\eta_a(x)=\langle \theta_a^{X},f_{\mathrm{ctx}}(x)\rangle-\kappa c_a\) is an article-dependent logit, whereas \(\xi_a(\tau)=\langle \theta_a^{U},e(\tau)\rangle\) is a segment-affinity logit.

In this benchmark, \(\beta\) changes only the OBS routing policy: \(\beta=0\) removes the latent-user term, and larger \(\beta\) makes routing more segment-specific. Reported regret is computed on the finite cached target \(q_{\mathrm{cache}}^{\mathrm{sm}}\) or \(q_{\mathrm{cache}}^{\mathrm{sh}}\). All reported numbers are means over 30 seeds (held-out split, OBS/EXP resamples, and EXP randomization combined). Appendix~\ref{app:real_setting} gives the exact slate construction, judged-reward aggregation, split details, the budget grid, and per-cell standard errors.

\begin{table}[tb]
\centering
\small
\setlength{\tabcolsep}{6pt}
\begin{tabular}{lccc}
\toprule
Method & Avg.\ rank & Top-3 & Excess (\%) \\
\midrule
\textbf{\CVCI{}} & \textbf{2.04} & \textbf{43} & \textbf{4.8} \\
\OBSOnly{} & 2.83 & 32 & 7.9 \\
\EXPOnly{} & 2.90 & 32 & 8.8 \\
\Grounded{} & 3.21 & 26 & 7.9 \\
\CVCIRes{} & 4.40 & 9 & 24.0 \\
\ProxyEXP{} & 5.62 & 2 & 77.6 \\
\bottomrule
\end{tabular}
\caption{Overall recommendation regret in the cached summarization benchmark. \emph{Avg.\ rank} is the average rank across the 48 settings; \emph{Top-3} is the count of settings (out of 48) in which a method ranked among the top three; \emph{Excess (\%)} is the mean percentage regret increase relative to the best method in each setting (lower is better in all three columns). Bolding marks the best per column; the gap between $\CVCI$ and the second-best method is specific to this benchmark instantiation. Standard errors and per-setting breakdowns are in Appendix~\ref{app:real_setting}.}
\label{tab:front_runner_summary}
\end{table}

\begin{table}[tb]
\centering
\small
\setlength{\tabcolsep}{6pt}
\renewcommand{\arraystretch}{1.05}
\begin{tabular}{@{}lcc cc@{}}
\toprule
\multicolumn{1}{c}{}
& \multicolumn{2}{c}{\(n_{\obs}=2{,}000\)}
& \multicolumn{2}{c}{\(n_{\obs}=20{,}000\)} \\
\cmidrule(lr){2-3}\cmidrule(lr){4-5}
Method & \(n_{\expd}=20\) & \(n_{\expd}=200\)
& \(n_{\expd}=20\) & \(n_{\expd}=200\) \\
\midrule
\CVCI{} & 0.0128 & 0.0128 & 0.0120 & 0.0120 \\
\OBSOnly{} & 0.0128 & 0.0128 & 0.0123 & 0.0123 \\
\Grounded{} & 0.0200 & 0.0133 & 0.0133 & 0.0135 \\
\EXPOnly{} & 0.0170 & 0.0121 & 0.0140 & 0.0118 \\
\ProxyEXP{} & 0.0330 & 0.0184 & 0.0365 & 0.0124 \\
\midrule
Best OBS-based method & 0.0128 & 0.0128 & 0.0120 & 0.0120 \\
\EXPOnly{} $-$ best OBS-based & $+0.0042$ & $-0.0007$ & $+0.0020$ & $-0.0002$ \\
\bottomrule
\end{tabular}
\caption{Recommendation regret (lower is better; mean over 30 seeds) under the smooth aggregate reward at different OBS and EXP budgets. The two summary rows below the divider are derived diagnostics, not estimator entries: ``Best OBS-based method'' is the per-cell minimum across $\CVCI$, $\OBSOnly$, $\Grounded$, and $\ProxyEXP$; ``$\EXPOnly{}$ $-$ best OBS-based'' is positive when OBS-assisted families help and negative when $\EXPOnly{}$ alone is competitive at that budget.}
\label{tab:summarization_budget_perf}
\end{table}

\subsubsection{Summarization: supervision scarcity and reward shape}

In the cached summarization benchmark, $\CVCI{}$ has the best average rank, top-3 coverage, and excess regret across 30 seeds and 48 settings (Table~\ref{tab:front_runner_summary}). The gap to the next family (about 0.04 in excess regret to $\OBSOnly{}$) is small relative to the spread of cell-level standard errors reported in Appendix~\ref{app:real_dr_diagnostics}, so this result reads as family-level evidence within this benchmark instantiation. The remaining tables explain that pattern. Table~\ref{tab:summarization_budget_perf} shows that increasing \(n_{\expd}\) lowers regret for the EXP-based estimators much more than increasing \(n_{\obs}\) under the smooth reward. That pattern matches the estimator view in Table~\ref{tab:theory_evidence_bridge}: once EXP and SIM identify the target, additional randomized labels can improve EXP-side reward fits. Table~\ref{tab:summarization_geometry_perf} then changes the target reward while holding the cached outputs fixed. The regret increase and the drop in rubric-linear fit serve as an empirical diagnostic of target alignment rather than a formal oracle statement.

In Table~\ref{tab:summarization_geometry_perf}, the two held-out \(R^2\) values are out-of-sample linear fits of the cached target surface using the auxiliary proxy representation and the raw rubric vector \(s(x,a)\), respectively. Table~\ref{tab:real_dr_summary} reports model-level value RMSE for DM and DR. In this benchmark, DR mainly helps OBS-heavy methods; DM remains stronger for the other families. Appendix~\ref{app:real_dr_diagnostics} gives the budget-specific breakdown and implementation details.

\begin{table}[tb]
\centering
\small
\setlength{\tabcolsep}{6pt}
\renewcommand{\arraystretch}{1.05}
\begin{tabular}{@{}lccc@{}}
\toprule
\multicolumn{4}{c}{Average regret over the four budget settings} \\
\cmidrule(lr){1-4}
Method & Smooth & Sharpened & \(\Delta_{\mathrm{shape}}\) \\
\midrule
\CVCI{} & 0.0126 & 0.0423 & $+0.0297$ \\
\OBSOnly{} & 0.0125 & 0.0477 & $+0.0352$ \\
\EXPOnly{} & 0.0147 & 0.0376 & $+0.0229$ \\
\ProxyEXP{} & 0.0267 & 0.0773 & $+0.0507$ \\
\midrule
\multicolumn{4}{c}{Held-out target-fit \(R^2\)} \\
\cmidrule(lr){1-4}
Quantity & Smooth & Sharpened & \(\Delta\) \\
\midrule
Proxy-bundle \(R^2\) & 0.161 & 0.171 & $+0.010$ \\
Rubric-linear \(R^2\) & 1.000 & 0.779 & $-0.221$ \\
\bottomrule
\end{tabular}
\caption{Recommendation regret (lower is better) and held-out target-fit $R^2$ under smooth versus sharpened reward maps, with the cached outputs held fixed. Sharpening the reward inflates $\Delta_{\mathrm{shape}}$ for every estimator family and drops the rubric-linear $R^2$ from $1.000$ to $0.779$, while the proxy-bundle $R^2$ stays low ($0.161\to0.171$). Reward shape changes the alignment between the cached target and the rubric-linear or proxy-bundle approximation classes.}
\label{tab:summarization_geometry_perf}
\end{table}

\begin{table}[tb]
\centering
\small
\setlength{\tabcolsep}{8pt}
\begin{tabular}{lccc}
\toprule
Method & DM RMSE & DR RMSE & $\Delta$ (DR$-$DM) \\
\midrule
\OBSOnly{} & 0.0479 & \textbf{0.0393} & $-0.0086$ \\
\CVCI{} & 0.0479 & 0.0540 & $+0.0061$ \\
\ProxyEXP{} & \textbf{0.0201} & 0.0560 & $+0.0359$ \\
\EXPOnly{} & 0.0835 & 0.1607 & $+0.0772$ \\
\bottomrule
\end{tabular}
\caption{Model-level value RMSE for DM and DR (lower is better), averaged at \(n_{\obs}=20{,}000\) and \(n_{\expd}\in\{100,200,800\}\) over 30 seeds. Negative $\Delta$ means DR helps over DM; positive $\Delta$ means DR hurts. DR helps OBS-heavy methods here and hurts EXP-side methods, consistent with finite-EXP residual-term variance dominating at small $n_{\expd}$.}
\label{tab:real_dr_summary}
\end{table}

\paragraph{Benchmark-instantiation reading.}
In this cached summarization benchmark, the raw judged dimensions already give a high held-out \(R^2\) for the target reward (Table~\ref{tab:summarization_geometry_perf}). Table~\ref{tab:summarization_budget_perf} shows two patterns: $\CVCI{}$ already attains the lowest regret at $n_{\expd}=20$ and stays near that level as $n_{\expd}$ grows, while $\EXPOnly{}$ and $\ProxyEXP{}$ improve substantially with more EXP labels and approach the $\CVCI{}$ floor at $n_{\expd}=200$. Additional EXP labels mainly help the EXP-anchored families catch up to the OBS-pooled $\CVCI{}$ floor rather than lower that floor.

\vspace{-9pt}

\subsection{Real-task cached coding benchmark with synthetic OBS/EXP resampling}
\label{sec:coding_public}

\vspace{-5pt}

The cached coding benchmark uses SWE-bench Verified issues \citep{Jimenez2024SWEBench,OpenAI2024SWEBenchVerified} and one cached candidate patch for each issue--agent pair (Table~\ref{tab:benchmark_snapshot}). Candidate patches come from BouncerBench \citep{Mathews2025BouncerBench}, and the agent baselines follow the SWE-agent setup \citep{YangEtAl2024SWEAgent}. The cached patches and program-test outcomes are real, and the OBS and EXP mechanisms are benchmark constructions over the cached pool. Evaluation averages over a fixed set of benchmark user types and reports regret on the finite cached target \(q_{\mathrm{cache}}\).

\begin{figure}[tb]
\centering
\includegraphics[width=0.95\linewidth]{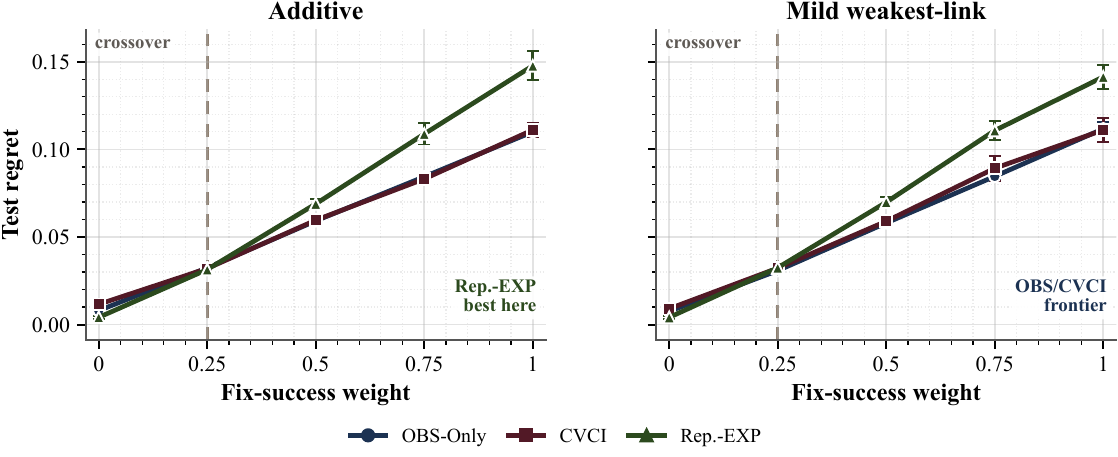}
\caption{Mean recommendation regret in the cached coding benchmark as the fix-success weight $\alpha_{\mathrm{fix}}$ increases (mean over 30 seeds; lower is better). Each panel fixes a value of the weakest-link weight $\omega_{\mathrm{weak}}$. The dashed line marks $\alpha_{\mathrm{fix}}=0.25$, the empirical crossover. At $\alpha_{\mathrm{fix}}<0.25$ the cached target depends primarily on the auxiliary patch-quality components $(c_2,c_3,c_4)$ and the representation-aided EXP family is among the lowest-regret families. For $\alpha_{\mathrm{fix}}>0.25$ the cached target depends primarily on $c_1$ (true fix success) and OBS-based families have lower regret. Top-two gaps are small in many cells (Appendix~\ref{app:coding_geometry}).}
\label{fig:coding_resolved_panels}
\end{figure}

Each cached patch is scored by four components \(c_1,c_2,c_3,c_4\), where \(c_1\) records true fix success and \(c_2,c_3,c_4\) summarize patch size, locality, and issue focus. These patch-quality components can disagree with true fix success, which is why changing the reward definition matters here.

Method rankings change as the reward shifts weight from patch quality to true fix success. The shift is controlled by \(\alpha_{\mathrm{fix}}\), which weights true fix success, and by \(\omega_{\mathrm{weak}}\), which determines whether the non-success components are aggregated additively or by a weakest-link rule. The first, \(\alpha_{\mathrm{fix}}\), controls how much weight the reward places on true fix success. The second, \(\omega_{\mathrm{weak}}\), controls whether the non-success components are aggregated additively or through a weakest-link rule.

For user type \(u\) with weights \(w_u=(w_{u1},w_{u2},w_{u3},w_{u4})\), normalized non-success weights are \(\bar w_{uk}=w_{uk}/(w_{u2}+w_{u3}+w_{u4})\) for \(k\in\{2,3,4\}\), and \(g_u(c_2,c_3,c_4)=(1-\omega_{\mathrm{weak}})\sum_{k=2}^4 \bar w_{uk}c_k+\omega_{\mathrm{weak}}\min\{c_2,c_3,c_4\}\). The benchmark utility for user type \(u\) is \(Y_u(x,a)=w_{u1}\,\alpha_{\mathrm{fix}}\,c_1(x,a)+(1-w_{u1})\,g_u\!\big(c_2(x,a),c_3(x,a),c_4(x,a)\big)\), and the evaluation target is the cached user-average surface \(q_{\mathrm{cache}}(x,a)=\E\!\left[Y_U(x,a)\mid X=x\right]\). \RepStruct{} maps each issue--patch pair into an auxiliary representation learned from observational patch statistics and fits EXP rewards on that representation. Appendix~\ref{app:coding_geometry} gives the public-data construction and the complete \(5\times 5\) results table.

With the data budget fixed, the ranking changes because the reward puts more or less weight on true fix success. The crossover occurs near \(\alpha_{\mathrm{fix}}\approx 0.25\): below that point \RepStruct{} is among the lowest-regret methods, and above it the best OBS-based methods overtake it. On the additive slice, the regret gap between \RepStruct{} and the better of \OBSStruct{} and \CVCIStruct{} moves from \(-0.0039\) at \(\alpha_{\mathrm{fix}}=0\) to \(+0.0003\) at \(\alpha_{\mathrm{fix}}=0.25\) and \(+0.0415\) at \(\alpha_{\mathrm{fix}}=1\). At \(\omega_{\mathrm{weak}}=0.25\), the corresponding values are \(-0.0034\), \(+0.0024\), and \(+0.0375\). Figure~\ref{fig:coding_pairwise_deltas} (left) shows the same crossover on the full \(5\times 5\) grid. This sweep tests whether the auxiliary patch statistics retain the target signal as the target moves from patch quality toward true fix success.

\begin{figure}[tb]
\centering
\includegraphics[width=0.95\linewidth]{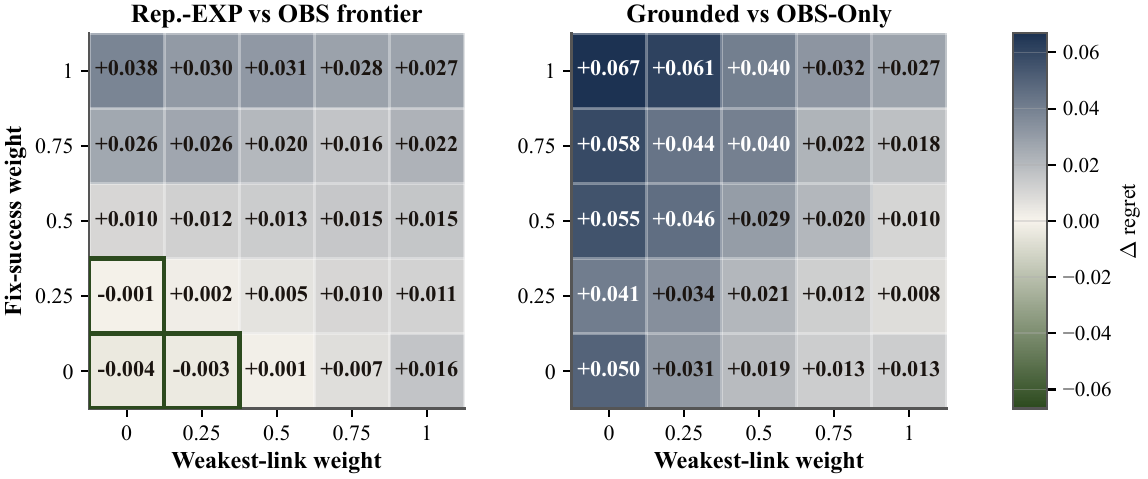}
\caption{Pairwise regret differences $\Delta\,\text{regret}=\text{regret}(\text{family A})-\text{regret}(\text{family B})$ across the $5\times 5$ grid of $(\alpha_{\mathrm{fix}},\omega_{\mathrm{weak}})$ in the cached coding benchmark instantiation; negative values mean family A has lower regret than family B in that cell. Left panel: \RepStruct{} minus the best OBS-based family (the per-cell minimum across $\OBSStruct$, $\CVCIStruct$, $\Grounded$). Right panel: \GroundedStruct{} minus $\OBSStruct$. Outlined cells in the left panel mark settings where \RepStruct{} has strictly lower regret than every OBS-based family; in this benchmark instantiation, weakest-link scoring shrinks the $\GroundedStruct$ vs $\OBSStruct$ gap but does not invert it.}
\label{fig:coding_pairwise_deltas}
\end{figure}

Held-out \(R^2\) comes from linear fits of the cached target using either the auxiliary patch statistics or \(c_1\) alone; Appendix~\ref{app:coding_geometry} gives the exact protocol. On the additive slice, the auxiliary-feature fit drops from \(R^2=0.347\) at \(\alpha_{\mathrm{fix}}=0\) to \(0.081\) at \(\alpha_{\mathrm{fix}}=1\), while the \(c_1\)-only fit rises from essentially zero to \(R^2=0.967\). The $R^2$ shift matches the change in method ranking: \RepStruct{} helps when the target still follows patch-quality features, but it loses that advantage once the reward is driven mainly by fix success.

Changing \(\omega_{\mathrm{weak}}\) changes only how the non-success components are combined. Its largest effect is on \GroundedStruct{}: at zero fix-success weight, the grounded-minus-OBS regret gap shrinks from about \(+0.0495\) under additive scoring to about \(+0.0126\) under the pure weakest-link rule. Figure~\ref{fig:coding_pairwise_deltas} (right) shows the same effect. Even so, \GroundedStruct{} remains worse than the best OBS-based method throughout this fixed-budget setting.

\paragraph{Benchmark-instantiation reading.}
In this fixed-budget coding probe over the BouncerBench cached patch pool, \RepStruct{} is the most competitive family in the regions where the auxiliary representation predicts the held-out cached target better than fix success alone (Figure~\ref{fig:coding_resolved_panels}, low-\(\alpha_{\mathrm{fix}}\) cells). Once \(c_1\) alone predicts the target at least as well, the OBS-based families overtake \RepStruct{} on held-out regret in the same probe (Figure~\ref{fig:coding_pairwise_deltas}). The pattern is a target-alignment diagnostic for this fixed-budget coding probe. Choosing among families in deployment would still require a live randomized trial.

%% file: sections/08_related_work.tex
\section{Related work}
\label{sec:related}

The paper combines three literatures: off-policy evaluation, hybrid estimation from observational and randomized data, and LLM evaluation. Each covers only part of the setting. The setting combines confounded logging on real tasks, replayable text mediators that can be regenerated for any candidate model, and expensive outcome labels that are unbiased only inside a small randomized experiment.

\paragraph{Off-policy evaluation.}
\DMFull{} and \DRFull{} are standard estimators in contextual-bandit and policy-evaluation settings \citep{Dudik2011DoublyRobust,Dudik2014DRStatSci,Swaminathan2015CRM,Jiang2016DROffPolicy,Thomas2016DataEfficient}, and they descend from the foundational doubly robust and missing-data lineage of \citet{RobinsRotnitzkyZhao1994} and \citet{BangRobins2005}. The OPE setting these papers analyze starts from logged action--outcome pairs in which outcomes are assumed unconfounded given context: the central technical issue is action coverage under the logging policy, and the recorded outcomes faithfully represent the causal target. In our setting, even for observed actions, the recorded outcomes are confounded because user-side factors influence both who selects which model and how its output is judged. Because SIM can regenerate outputs for any model on any held-out context, DM and DR enter only after the reward surface itself has been estimated. We use SIM and EXP to identify the causal scoring rule and then use OBS to reduce variance.

\paragraph{Experimental grounding and hybrid estimation.}
The closest methodological antecedents are experimental grounding \citep{Kallus2018ExperimentalGrounding}, data-fusion and transportability formalisms \citep{PearlBareinboim2011Transportability,BareinboimPearl2016DataFusion}, and trial-to-target generalizability \citep{ColeStuart2010}. A related line of work studies how to combine large biased samples with smaller randomized ones \citep{Rosenman2023Shrinkage,Cheng2021Adaptive,Lin2025PowerLikelihood,Yang2025CVCI,Colnet2024Review}. That literature asks when observational signal can be used to reduce variance without giving up the causal credibility of the randomized sample; in those papers, the unit-level outcome is observed on both samples and the modeling work is done on a fixed outcome label. The estimator families in Section~\ref{sec:estimators} are concrete instantiations of the same bias--variance question. Our setting differs in one specific way: the target outcome is not directly observed even on the randomized sample because the unit ``output'' must first be regenerated by SIM and then scored. Identification in our setting depends on SIM validity (Assumption~A2) together with EXP randomization.

\paragraph{Summarization and LLM-based evaluation.}
Our experiments use CNN/DailyMail contexts \citep{Hermann2015CNN,Nallapati2016Sum,See2017Pointer,Stiennon2020SummarizeHF} and rely on LLM-judged rubric scores in the pool-based study. Recent LLM-evaluation work covers benchmark design \citep{Liang2023HELM}, judge reliability and pairwise preference platforms \citep{Zheng2023LLMJudge,Chiang2024ChatbotArena}, and real-world usage logs \citep{Zheng2023LMSYSChat1M,Zhao2024WildChat}. It also documents biases in automatic evaluators, including positional unfairness \citep{WangEtAl2023FairEvaluators}, length bias \citep{Dubois2024LengthControlledAlpacaEval}, and self-preference \citep{PanicksseryBowmanFeng2024}, and studies panel-based mitigations \citep{VergaEtAl2024Juries}. Preference data is central in alignment pipelines, both with explicit learned reward models \citep{Christiano2017HumanPreferences,Stiennon2020SummarizeHF,Ouyang2022InstructGPT} and with implicit reward formulations such as direct preference optimization \citep{Rafailov2023DPO}. The LLM-evaluation work cited above primarily treats judged scores as outcome data and studies how reliable those scores are and how to average them. We use cached outputs and judged scores as inputs to a causal evaluation problem; they become the causal target only after EXP-identified scoring.

A complementary line of work in causal NLP studies how text can serve as treatment, outcome, mediator, or proxy for unobserved confounders \citep{FederEtAl2022CausalNLP}; OBS-derived auxiliary labels in our setting play exactly the proxy role identified there.

Taken together, the setting sits across the three literatures rather than fitting cleanly inside any one of them.

%% file: sections/07_discussion.tex
\section{Discussion and Future Work}
\label{sec:discussion}

Ranking diagnostics can diverge sharply from recommendation regret. In the cached summarization benchmark, \Grounded{} wins every top-1 and top-3 ranking setting but wins regret in only 6 settings. Regret is driven by a small number of costly recommendation mistakes that average ranking metrics do not isolate. The benchmarks therefore center $\mathrm{Regret}_{\mathrm{test}}$, with context-action RMSE and model-level RMSE as secondary diagnostics.

\paragraph{What our cached benchmarks are and are not.}
The two real-task benchmarks use real candidate model outputs and real LLM-judged rubric or program-test scores on those outputs. Their OBS and EXP mechanisms are benchmark constructions over the cached pool. The empirical comparisons therefore diagnose how the estimator families behave under controlled synthetic OBS/EXP resampling on real tasks.

\paragraph{Scope of the empirical benchmarks.}
The empirical scope is deliberately narrow: cached summarization isolates supervision scarcity, and cached coding isolates reward geometry. Family-level wins in these benchmarks are benchmark-specific diagnostics.

\paragraph{Methodological implications.}
Estimator choice depends on what limits performance in the benchmark at hand. In practice, the key checks are whether more EXP labels materially improve EXP-based methods and whether a structured representation predicts the held-out target better than a simple success-only baseline. Deployment decisions should still be confirmed by live randomized evaluation.

\paragraph{Limitations and future work.}
These conclusions are specific to the task families and benchmark constructions studied here. Future work includes broader tasks, repeated mediator draws, naturally observed deployment logs paired with in-deployment randomized trials, and multi-round human--agent interaction, which would require relaxing or replacing Assumption~A4.

%% file: sections/09_conclusion.tex
\section{Conclusion}
\label{sec:conclusion}

Offline evaluation separates into identification and post-identification estimation. In the three-source design, the simulator and the randomized experiment together suffice to identify causal model values; the observational log helps only with estimation afterward, when its signal genuinely matches the causal target.

Across the controlled semi-synthetic validation and the two real-task cached benchmark constructions, different estimator families win in different regimes on held-out recommendation regret. The two benchmark families point to the same practical lesson: which estimator works best depends on both the amount of experimental supervision and on what the reward is actually trying to capture.

The results apply to one-shot cached-evaluation settings, not yet to richer interactive settings with repeated feedback and broader population-level uncertainty. Future work should extend the framework beyond one-shot cached evaluations to richer interactive settings and live deployment with randomized evaluation.

%% file: sections/A_additional_details.tex
\section{Additional details}
\label{sec:appendix}

\subsection{Identification remarks}
\label{app:id_remarks}

\paragraph{SIM resolves generation; EXP resolves scoring.}
SIM regenerates the mediator $M$ under $\doop(A=a)$, while EXP supplies the outcome labels needed to score that mediator.
So SIM handles counterfactual generation and EXP identifies causal scoring.
The randomized sample is what identifies the scoring rule $r^\star(x,m)$ for a realized context--output pair.

\paragraph{Why observational outcomes target a different object.}
In deployment, $A$ depends on the latent variable $U$, and $U$ also affects $Y$ directly.
Conditioning on $(X,M)$ therefore leaves self-selection bias in place, so in general
\[
\E[Y\mid X=x,M=m,\obs]\neq r^\star(x,m).
\]
This is why the additional assumptions in Section~\ref{sec:estimators} are estimator-side assumptions for using OBS, while identification itself comes from SIM and EXP.

\subsection{Estimator details and concrete implementations}
\label{app:estimator_details}

\paragraph{Exact reward-model objectives.}
This subsection records the exact fitting problems and benchmark-specific implementation details deferred from Section~\ref{sec:estimators}.
\EXPOnly{} fits
\[
(\widehat w_{\expd},\widehat b_{\expd})\in\arg\min_{w,b}\sum_{j\in\expd}\big(Y_j-(w^\top\varphi(z_j)+b)\big)^2 + \alpha\|w\|_2^2,
\]
with prediction rule $\widehat r_{\expd}(z)=\clip_{[0,1]}(\widehat w_{\expd}^\top\varphi(z)+\widehat b_{\expd})$.

\ProxyEXP{} fits the same ridge head on the proxy representation:
\[
(\widehat w_{\psi},\widehat b_{\psi})\in\arg\min_{w,b}\sum_{j\in\expd}\big(Y_j-(w^\top\psi(z_j)+b)\big)^2 + \alpha\|w\|_2^2,
\]
with $\widehat r_{\psi}(z)=\clip_{[0,1]}(\widehat w_{\psi}^\top\psi(z)+\widehat b_{\psi})$.

\OBSOnly{} fits the raw-feature ridge model on OBS outcomes,
\[
(\widehat w_{\obs},\widehat b_{\obs})\in\arg\min_{w,b}\sum_{i\in\obs}\big(Y_i-(w^\top\varphi(z_i)+b)\big)^2 + \alpha\|w\|_2^2,
\]
and predicts $f_{\obs}(z)=\clip_{[0,1]}(\widehat w_{\obs}^\top\varphi(z)+\widehat b_{\obs})$.

For \Grounded{}, define EXP residual targets $\Delta_j=f_{\obs}(z_j)-Y_j$ and fit
\[
(\widehat\theta,\widehat c)\in \arg\min_{\theta,c}\sum_{j\in\expd}\big(\Delta_j-(\theta^\top\psi(z_j)+c)\big)^2 + \lambda_\theta\|\theta\|_2^2.
\]
The resulting predictor subtracts the fitted proxy-side correction from the OBS baseline and clips the result to $[0,1]$. In the real benchmarks, the grounded family keeps the same outer form but replaces the single linear proxy correction with a small library of proxy-basis corrections built on the heuristic auxiliary proxy. Appendix~\ref{app:grounded_real_upgrade} compares this richer correction class with a single-linear baseline and with a weak pooled-anchor variant.

\CVCI{} pools the OBS and EXP regression objectives:
\begin{equation}
\begin{aligned}
(\widehat w_{\lambda},\widehat b_{\lambda})\in\arg\min_{w,b}\Bigg[
&(1-\lambda)\,\frac{1}{n_{\expd}}
\sum_{j\in\expd}\big(Y_j-(w^\top\varphi(z_j)+b)\big)^2 \\
&\qquad
+ \lambda\,\frac{1}{n_{\obs}}
\sum_{i\in\obs}\big(Y_i-(w^\top\varphi(z_i)+b)\big)^2
\Bigg]
+ \alpha\|w\|_2^2,
\end{aligned}
\label{eq:cvci_objective}
\end{equation}
with $\widehat r_{\lambda}(z)=\clip_{[0,1]}(\widehat w_{\lambda}^\top\varphi(z)+\widehat b_{\lambda})$.
Here $\lambda=1$ is the OBS-only endpoint and $\lambda=0$ is the EXP-only endpoint.
The final predictor refits \eqref{eq:cvci_objective} on the full OBS and EXP data at the value of $\lambda$ selected by Appendix~\ref{app:agent_cv}.

\CVCIRes{} residualizes around $f_{\obs}$ using the residual target $Y_j-f_{\obs}(z_j)$ and pools only that residual fit in proxy space:
\[
\begin{aligned}
(\widehat\theta_{\lambda},\widehat c_{\lambda})\in\arg\min_{\theta,c}\Bigg[
&(1-\lambda)\,\frac{1}{n_{\expd}}
\sum_{j\in\expd}\big(Y_j-f_{\obs}(z_j)-(\theta^\top\psi(z_j)+c)\big)^2 \\
&\qquad
+ \lambda\,\frac{1}{n_{\obs}}
\sum_{i\in\obs}\big(Y_i-f_{\obs}(z_i)-(\theta^\top\psi(z_i)+c)\big)^2
\Bigg]
+ \alpha_{\psi}\|\theta\|_2^2,
\end{aligned}
\]
with $\widehat r_{\mathrm{res},\lambda}(z)=\clip_{[0,1]}\big(f_{\obs}(z)+\widehat\theta_{\lambda}^\top\psi(z)+\widehat c_{\lambda}\big)$.
In the experiments, both the pooling weight $\lambda$ and the residual ridge penalty $\alpha_{\psi}$ are selected by Appendix~\ref{app:agent_cv}; when residual-CVCI candidates are tied within numerical tolerance, the implementation resolves the tie toward the more regularized candidate.

\paragraph{How each estimator family is implemented.}
Table~\ref{tab:concrete_impls} shows how each estimator family is implemented in the appendix validation and in the two real benchmarks. The family names are the same across settings, but the proxy construction and EXP-side tuning differ by benchmark.

\begin{table}[t]
\centering
\scriptsize
\setlength{\tabcolsep}{4pt}
\caption{How each estimator family is implemented in the appendix validation and the real benchmarks. Family names are shared across settings, but the proxy channel and EXP-side tuning differ by benchmark.}
\label{tab:concrete_impls}
\begin{tabularx}{\textwidth}{@{}p{1.60cm}>{\hsize=0.78\hsize\linewidth=\hsize}X>{\hsize=1.22\hsize\linewidth=\hsize}X p{2.40cm}@{}}
\toprule
Paper family & Semi-synthetic instantiation & Real-benchmark instantiation & EXP-side tuning \\
\midrule
\EXPOnly{} & Raw-feature ridge fit on EXP only & Same family on cached real outputs & Fixed defaults \\
\OBSOnly{} & Raw-feature ridge fit on OBS outcomes only & Same family on cached real outputs & None \\
\ProxyEXP{} & EXP-only reward head on a proxy learned from OBS auxiliary labels & EXP-only reward head on a heuristic proxy learned from cached trajectory metadata & Fixed defaults \\
\Grounded{} & OBS baseline plus low-dimensional proxy correction & OBS baseline plus richer proxy-basis correction on the heuristic auxiliary proxy & Small EXP-only tuning \\
\CVCI{} & Direct pooled OBS/EXP fit in raw text features & Same family with model-level EXP cross-validation for the pooling weight & Benchmark-specific EXP tuning \\
\CVCIRes{} & OBS baseline plus pooled residual fit in proxy space & Same family with the heuristic auxiliary proxy and model-level EXP cross-validation & Benchmark-specific EXP tuning \\
\bottomrule
\end{tabularx}
\end{table}

Throughout Section~\ref{sec:real}, \Grounded{} refers to the rich proxy-basis version in Table~\ref{tab:concrete_impls}. Appendix~\ref{app:grounded_real_upgrade} compares it with the single-linear baseline and the pooled-anchor variant.

\subsection{Grounded-family implementations for the real cached benchmark}
\label{app:grounded_real_upgrade}

\paragraph{A common grounded template.}
All real-benchmark grounded variants start from the same decomposition: fit an OBS baseline $f_{\obs}$ from logged outcomes, then use the heuristic auxiliary proxy to estimate a correction that is subtracted from that baseline.
The variants differ along two design axes:
\begin{enumerate}[leftmargin=*]
\item the correction class used on the proxy representation, and
\item whether OBS enters the second stage only through $f_{\obs}$ or also through a small pooled anchor after the correction direction has been fixed.
\end{enumerate}
The grounded-family comparison instead organizes three principled instantiations of the same baseline-minus-correction idea.

\paragraph{Variant 1: single-linear grounded correction.}
This is the most straightforward grounded instantiation.
It uses the heuristic auxiliary proxy, keeps a single linear correction family, and tunes the correction strength on EXP.
For the auxiliary real summarization comparison over \(\beta\in\{0,0.2,0.5,0.8,0.9,0.99\}\), \(n_{\obs}\in\{2{,}000,20{,}000\}\), and \(n_{\expd}\in\{20,100\}\), we report the version of this linear class tuned by model-level EXP cross-validation, which serves as the matched linear baseline for the richer grounded variants.
That single-linear class uses the predictor
\[
\widehat r_{\mathrm{lin},\alpha}(z)=\clip_{[0,1]}\!\Big(f_{\obs}(z)-\alpha\big(\widehat\theta^\top\psi_{\mathrm{aux}}(z)+\widehat c\big)\Big),
\]
where $\psi_{\mathrm{aux}}$ is the learned heuristic auxiliary proxy and the correction is linear in that proxy.

\paragraph{Variant 2: \Grounded{} as a rich but still small proxy-side correction class.}
The main-text \Grounded{} instantiation keeps the same baseline-minus-correction outer form, but replaces the single linear correction by a small library of proxy-basis corrections.
Write $\tilde\psi_{\mathrm{aux}}(z)\in\R^{16}$ for the standardized 16-dimensional SVD compression of the learned heuristic auxiliary proxy.
For
\[
B\in\{B_{\mathrm{id}},B_{\mathrm{poly}}\},
\qquad
B_{\mathrm{id}}(u)=u,
\qquad
B_{\mathrm{poly}}(u)=\big[u,\ u\odot u\big],
\]
and ridge level $\tau\in\{10^{-2},1,100\}$, the rich grounded direction fits
\[
(\widehat\theta_{B,\tau},\widehat c_{B,\tau})
\in
\arg\min_{\theta,c}
\sum_{j\in\expd}
\big(\Delta_j-(\theta^\top B(\tilde\psi_{\mathrm{aux}}(z_j))+c)\big)^2
\;+\;
\tau\|\theta\|_2^2,
\qquad
\Delta_j=f_{\obs}(z_j)-Y_j,
\]
and predicts
\[
\widehat r_{\mathrm{rich},B,\alpha}(z)
=
\clip_{[0,1]}\!\Big(f_{\obs}(z)-\alpha\big(\widehat\theta_{B,\tau}^\top B(\tilde\psi_{\mathrm{aux}}(z))+\widehat c_{B,\tau}\big)\Big).
\]
The final implementation selects $(B,\tau,\alpha)$ by agent-CV on EXP.
Relative to the linear grounded class, this replaces the correction family
\[
\mathcal H_{\mathrm{lin}}
=
\Big\{u\mapsto \theta^\top u+c\Big\}
\]
by
\[
\mathcal H_{\mathrm{rich}}
=
\Big\{u\mapsto \theta^\top B(u)+c:\ B\in\{B_{\mathrm{id}},B_{\mathrm{poly}}\}\Big\},
\]
so $\mathcal H_{\mathrm{lin}}\subseteq \mathcal H_{\mathrm{rich}}$.
The corresponding grounded predictor class therefore expands the proxy-side approximation space while still learning the correction direction from EXP only.

\paragraph{Variant 3: pooled-anchor grounded.}
The appendix pooled-anchor variant keeps the rich EXP-only correction direction but adds a weak pooled anchor in the second stage.
Let
\[
\widehat c_{B,\tau}(z)=\widehat\theta_{B,\tau}^\top B(\tilde\psi_{\mathrm{aux}}(z))+\widehat c_{B,\tau}
\]
denote the grounded correction direction learned from EXP.
The pooled-anchor predictor is
\[
\widehat r_{\mathrm{anchor},\lambda}(z)
=
\clip_{[0,1]}\!\Big(\widehat b_{\lambda}\,f_{\obs}(z)-\widehat\alpha_{\lambda}\,\widehat c_{B,\tau}(z)\Big),
\]
where $(\widehat b_{\lambda},\widehat\alpha_{\lambda})$ are obtained from the anchored objective
\[
\begin{aligned}
(\widehat b_{\lambda},\widehat\alpha_{\lambda})\in\arg\min_{b,\alpha}\Bigg[
&\lambda\,\frac{1}{n_{\obs}}\sum_{i\in\obs}
\big(Y_i-(b f_{\obs}^{\mathrm{cf}}(z_i)-\alpha \widehat c_{B,\tau}(z_i))\big)^2 \\
&\qquad
+ (1-\lambda)\,\frac{1}{n_{\expd}}\sum_{j\in\expd}
\big(Y_j-(b f_{\obs}(z_j)-\alpha \widehat c_{B,\tau}(z_j))\big)^2
\Bigg] \\
&\qquad
+ \rho_b(b-1)^2+\rho_{\alpha}(\alpha-1)^2.
\end{aligned}
\]
Here $f_{\obs}^{\mathrm{cf}}$ is a cross-fitted OBS baseline prediction.
This variant pools only two calibration coefficients after EXP has already fixed the correction direction.

\begin{table}[t]
\centering
\scriptsize
\setlength{\tabcolsep}{3pt}
\renewcommand{\arraystretch}{1.08}
\caption{
Grounded-family implementations for the real cached benchmark.
The table compares the single-linear grounded baseline, the richer proxy-basis grounded instantiation used in the main text, and a weak pooled-anchor variant built on the same rich correction direction.
`Regret wins' counts wins among the seven methods obtained by adding the pooled-anchor variant to the six main-text methods.
}
\label{tab:grounded_upgrade_compare}
\begin{tabularx}{\textwidth}{@{}>{\raggedright\arraybackslash}p{2.25cm}>{\raggedright\arraybackslash}X>{\raggedright\arraybackslash}p{2.05cm}>{\raggedright\arraybackslash}p{2.15cm}S[table-format=1.0]S[table-format=1.4]S[table-format=1.3]S[table-format=1.3]@{}}
\toprule
Variant & Correction family & \shortstack[l]{Second-stage\\OBS use} & Tuning & \multicolumn{1}{c}{\shortstack{Regret\\wins}} & \multicolumn{1}{c}{\shortstack{Macro\\regret}} & \multicolumn{1}{c}{\shortstack{Mean\\top-1}} & \multicolumn{1}{c}{\shortstack{Mean\\top-3}} \\
\midrule
Single-linear correction baseline & Single linear correction on the heuristic auxiliary proxy & None beyond the baseline $f_{\obs}$ & Model-level EXP cross-validation for $\alpha_{\mathrm{corr}}$ & 1 & 0.0357 & 0.628 & 0.646 \\
\Grounded{} rich proxy-basis correction & Rich proxy-basis correction on standardized 16-dimensional heuristic auxiliary features with $B\in\{\mathrm{id},\mathrm{poly2}\}$ and $\tau\in\{10^{-2},1,100\}$ & None beyond the baseline $f_{\obs}$ & Model-level EXP cross-validation over $(B,\tau,\alpha_{\mathrm{corr}})$ & 6 & 0.0318 & 0.736 & 0.736 \\
Rich correction with weak pooled anchor & Same rich EXP-only correction direction as \Grounded{}, followed by anchored coefficient pooling & OBS enters only through the pooled coefficients $(b,\alpha)$ & Model-level EXP cross-validation over $(B,\tau,\lambda)$ with closed-form anchored $(b,\alpha)$ & 4 & 0.0326 & 0.744 & 0.744 \\
\bottomrule
\end{tabularx}
\end{table}

\paragraph{Why these implementations should behave differently.}
The three variants place their modeling flexibility in different parts of the decomposition.
The linear grounded baseline puts the entire second stage inside the smallest class $\mathcal H_{\mathrm{lin}}$.
So it is the most conservative option, but it can underfit whenever the discrepancy between $f_{\obs}$ and the causal target is only approximately linear in the proxy.
Because \Grounded{} replaces $\mathcal H_{\mathrm{lin}}$ by the strictly larger class $\mathcal H_{\mathrm{rich}}\supseteq \mathcal H_{\mathrm{lin}}$ while still estimating the correction direction from EXP only, it should help precisely when the residual mismatch is slightly nonlinear in the proxy yet still low-dimensional.
The pooled-anchor variant changes a different part of the problem.
It keeps the correction direction fixed at the grounded solution and lets OBS influence only the coarse scaling coefficients $(b,\alpha)$.
It should help when the main issue is the relative calibration of the baseline and the correction, especially once EXP is large enough to identify a stable correction direction.

\paragraph{Results on the \((\beta,n_{\obs},n_{\expd})\) comparison.}
Relative to the matched single-linear baseline, \Grounded{} lowers macro regret from $0.0357$ to $0.0318$, increases regret wins from $1$ to $6$, and has lower mean regret in $15$ of the $24$ \((\beta,n_{\obs},n_{\expd})\) settings. The gains are concentrated in \((2000,20)\) and \((20000,100)\), where \Grounded{} wins all six beta values, and in \((20000,20)\), where it wins three of six.

The pooled-anchor variant behaves differently. It beats \Grounded{} in $11$ of the $24$ settings and has the strongest average ranking metrics in this three-way comparison, but its macro regret is $0.0326$ and it wins regret in only four settings, all with $n_{\expd}=100$.
Those wins occur only when $n_{\expd}=100$:
\[
(\beta,n_{\obs},n_{\expd})\in\{(0,2000,100),\ (0.5,2000,100),\ (0.8,20000,100),\ (0.9,20000,100)\}.
\]
The selected pooling weight is usually small: across $720$ fitted models, $\lambda_\star\le 0.03$ in $565$ fits and $\lambda_\star=0$ in $264$ fits. In this comparison, the pooled anchor changes calibration more than it changes the regret winner.

\subsection{Value estimators and empirical benchmark metrics}
\label{app:value_estimators}

The reward models in Section~\ref{sec:estimators} estimate a conditional outcome surface.
To obtain generative-model-level causal values, we combine those reward predictions with SIM.
These DM / DR aggregators are used for the semi-synthetic generative-model values.
They sit alongside the direct held-out-grid metrics in the real cached benchmark.

Let $\{x_i\}_{i=1}^{n_{\mathrm{eval}}}$ be evaluation contexts and let $m_{i,a,b}\sim p_{\simu}(\cdot\mid x_i,a)$ denote SIM-generated mediators.
The default SIM plug-in estimator is
\[
\widehat\mu^{\mathrm{DM}}(a)=\frac{1}{n_{\mathrm{eval}}}\sum_{i=1}^{n_{\mathrm{eval}}}\frac{1}{B}\sum_{b=1}^B \widehat r(x_i,m_{i,a,b}).
\]
This simply averages predicted rewards over SIM-generated outputs.

A doubly robust variant augments the plug-in estimator with a residual term computed from EXP.
Because EXP randomizes generative-model choice uniformly, $p_{\expd}(a)=1/|\mathcal{A}_{\expd}|$.
Define
\[
\widehat q(x,a)=\frac{1}{B_{\mathrm{DR}}}\sum_{b=1}^{B_{\mathrm{DR}}}\widehat r\big(x,m_{a,b}\big),
\qquad m_{a,b}\sim p_{\simu}(\cdot\mid x,a),
\]
as the Monte Carlo approximation to the SIM-integrated reward.
The estimator is
\[
\widehat\mu^{\mathrm{DR}}(a)=
\frac{1}{n_{\expd}}\sum_{j=1}^{n_{\expd}}\widehat q(X_j,a)
+
\frac{1}{n_{\expd}}\sum_{j=1}^{n_{\expd}}
\frac{\Ind\{A_j=a\}}{p_{\expd}(a)}\big(Y_j-\widehat r(X_j,M_j)\big).
\]
The first term is the SIM-based plug-in estimate.
The second term corrects it using the randomized residual observed in EXP.
For reward models trained on EXP, we use cross-fitting: EXP is partitioned into $K_{\mathrm{cf}}$ folds, and the reward model used in the residual term is trained on the remaining $K_{\mathrm{cf}}-1$ folds.

\paragraph{Role of DM / DR versus direct cached-benchmark metrics.}
In the appendix semi-synthetic validation, we center held-out recommendation regret computed against the generator's true interventional rewards.
The same runs also produce model-level value diagnostics such as the agent-level RMSE of $\widehat\mu^{\mathrm{DM}}(a)$ or $\widehat\mu^{\mathrm{DR}}(a)$, which remain useful for diagnosing reward-surface fit.
In the real cached benchmark, because the full held-out cached test grid is available, we additionally evaluate the fitted reward surface directly on that grid.
Let $\mathcal{X}_{\mathrm{test}}$ be the held-out test contexts and $\mathcal{G}_{\mathrm{test}}=\mathcal{X}_{\mathrm{test}}\times\mathcal{A}$ the held-out cached grid.
For each $(x,a)\in\mathcal{G}_{\mathrm{test}}$, let $q_{\mathrm{cache}}(x,a)$ be the judged scalar reward of the single cached held-out trajectory for context $x$ and agent $a$, and let $\widehat q(x,a)$ be the fitted reward model evaluated on that cached trajectory.
Define the benchmark-level agent aggregates
\[
\bar q_{\mathrm{cache}}(a)=\frac{1}{|\mathcal{X}_{\mathrm{test}}|}\sum_{x\in\mathcal{X}_{\mathrm{test}}} q_{\mathrm{cache}}(x,a),
\qquad
\bar q_{\mathrm{pred}}(a)=\frac{1}{|\mathcal{X}_{\mathrm{test}}|}\sum_{x\in\mathcal{X}_{\mathrm{test}}} \widehat q(x,a).
\]
The real-benchmark empirical metrics are then
\[
\mathrm{RMSE}_{xa}
=
\left(
\frac{1}{|\mathcal{G}_{\mathrm{test}}|}
\sum_{(x,a)\in\mathcal{G}_{\mathrm{test}}}
\big(\widehat q(x,a)-q_{\mathrm{cache}}(x,a)\big)^2
\right)^{1/2},
\]
\[
\mathrm{RMSE}_{\mathrm{agent}}
=
\left(
\frac{1}{|\mathcal{A}|}
\sum_{a\in\mathcal{A}}
\big(\bar q_{\mathrm{pred}}(a)-\bar q_{\mathrm{cache}}(a)\big)^2
\right)^{1/2},
\qquad
\widehat\pi(x)=\arg\max_{a\in\mathcal{A}}\widehat q(x,a),
\]
\[
\mathrm{Regret}_{\mathrm{test}}
=
\frac{1}{|\mathcal{X}_{\mathrm{test}}|}
\sum_{x\in\mathcal{X}_{\mathrm{test}}}
\Big[\max_{a\in\mathcal{A}} q_{\mathrm{cache}}(x,a)-q_{\mathrm{cache}}(x,\widehat\pi(x))\Big].
\]
These real-benchmark metrics summarize the finite held-out cached grid.
The natural next extension is repeated cached replays / repeated judged trajectories per $(x,a)$ so that the benchmark can estimate latent $q(x,a)$ more directly.

\subsection{Empirical role of DM / DR on the controlled benchmark}
\label{app:dr_diagnostics}

The same semi-synthetic runs also let us compare DM and DR on model-level value estimation.
These diagnostics answer a narrower question than the regret tables: when does randomized residual correction improve the SIM plug-in value estimate?

\begin{center}
\captionsetup{hypcap=false}
\scriptsize
\setlength{\tabcolsep}{5pt}
\renewcommand{\arraystretch}{1.1}
\begin{tabular}{lccc}
\toprule
Method & DM RMSE & DR RMSE & $\Delta$ (DR$-$DM) \\
\midrule
\EXPOnly{} & 0.0303 & 0.0432 & +0.0128 \\
\OBSOnly{} & 0.0268 & 0.0409 & +0.0141 \\
\CVCI{} & 0.0212 & 0.0399 & +0.0187 \\
\ProxyEXP{} & 0.0263 & 0.0424 & +0.0161 \\
\bottomrule
\end{tabular}
\captionof{table}{Average agent-level value RMSE across the synthetic \((\beta,n_{\obs},n_{\expd})\) grid with \(\beta\in\{0,0.2,0.5,0.8,0.9,0.99\}\), \(n_{\obs}\in\{2{,}000,20{,}000\}\), and \(n_{\expd}\in\{20,100\}\). Positive $\Delta$ means that replacing DM by DR increases error.}
\label{tab:dr_negative_control}
\end{center}

Table~\ref{tab:dr_negative_control} shows that, on this synthetic \((\beta,n_{\obs},n_{\expd})\) grid, the SIM plug-in estimator has lower average agent-level value RMSE than DR for every representative method shown. Across these \(24\) settings, DR does not improve model-level value estimation on average.

\begin{center}
\captionsetup{hypcap=false}
\includegraphics[width=0.7\linewidth]{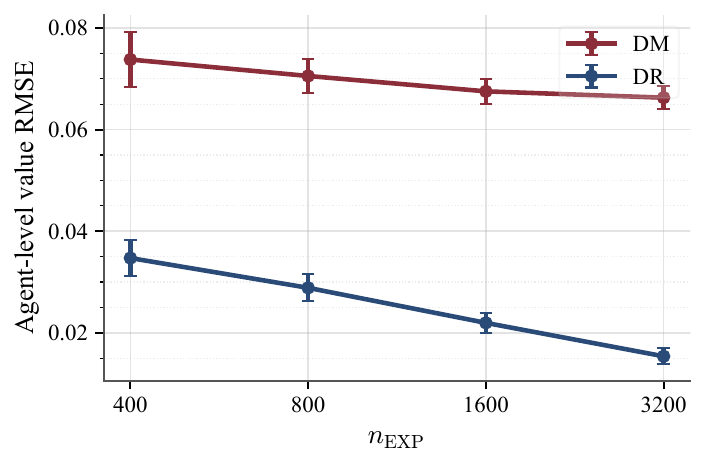}
\captionof{figure}{Targeted DM/DR diagnostic under strong self-selection, intentional reward-model underfit, and increasing EXP budget with fixed OBS budget. In this regime, DR materially reduces agent-level value RMSE.}
\label{fig:dr_targeted_curve}
\end{center}

Figure~\ref{fig:dr_targeted_curve} shows a different case, with stronger self-selection, intentional reward-model underfit, lower outcome noise, and increasing EXP budget at fixed OBS budget. In that setting, DR lowers agent-level RMSE from \(0.0738\) to \(0.0347\) at \(n_{\expd}=400\), and from \(0.0663\) to \(0.0154\) at \(n_{\expd}=3200\). This is a narrower result about model-level value estimation; it does not change the main regret comparisons in the real summarization and coding benchmarks.

\subsection{Empirical role of DM / DR on the real summarization benchmark}
\label{app:real_dr_diagnostics}

The real summarization benchmark also allows a DM/DR comparison for model-level value estimation. Here the target is the benchmark-level agent average over the held-out cached test contexts, and DR uses the known uniform EXP propensity together with cross-fitting for methods trained on EXP. This is a model-level value check built on top of the fitted reward model, not a replacement for the regret comparisons in the main text.

\begin{center}
\captionsetup{hypcap=false}
\scriptsize
\setlength{\tabcolsep}{5pt}
\renewcommand{\arraystretch}{1.1}
\begin{tabular}{llccc}
\toprule
Method & $n_{\expd}$ & DM RMSE & DR RMSE & $\Delta$ (DR$-$DM) \\
\midrule
\OBSOnly{} & 100 & 0.0479 & 0.0459 & -0.0021 \\
\OBSOnly{} & 200 & 0.0479 & 0.0410 & -0.0070 \\
\OBSOnly{} & 800 & 0.0479 & 0.0310 & -0.0169 \\
\midrule
\CVCI{} & 100 & 0.0479 & 0.0724 & +0.0245 \\
\CVCI{} & 200 & 0.0479 & 0.0572 & +0.0093 \\
\CVCI{} & 800 & 0.0479 & 0.0325 & -0.0154 \\
\midrule
\ProxyEXP{} & 100 & 0.0226 & 0.0749 & +0.0522 \\
\ProxyEXP{} & 200 & 0.0187 & 0.0634 & +0.0447 \\
\ProxyEXP{} & 800 & 0.0190 & 0.0298 & +0.0108 \\
\midrule
\EXPOnly{} & 100 & 0.1092 & 0.2495 & +0.1403 \\
\EXPOnly{} & 200 & 0.0809 & 0.1528 & +0.0718 \\
\EXPOnly{} & 800 & 0.0604 & 0.0798 & +0.0194 \\
\bottomrule
\end{tabular}
\captionof{table}{Budget-specific DM/DR agent-level value RMSE on the real summarization benchmark at \(n_{\obs}=20{,}000\). Negative $\Delta$ means that replacing DM by DR reduces error.}
\label{tab:real_dr_breakdown}
\end{center}

Table~\ref{tab:real_dr_breakdown} makes the heterogeneity explicit.
\OBSOnly{} improves in \(23\) of the \(30\) benchmark cells and becomes more favorable as \(n_{\expd}\) grows.
\CVCI{} changes sign only at \(n_{\expd}=800\), where the residual term is large enough to help more than it hurts.
The EXP-side methods behave differently: \ProxyEXP{} worsens in \(28\) of \(30\) cells, and \EXPOnly{} worsens in \(27\) of \(30\).

The natural interpretation is that DR is correcting the remaining OBS-side bias in model-level values rather than uniformly improving all fitted reward models.
On this benchmark, that makes it most useful for the OBS-heavy baseline, conditionally useful for \CVCI{} once the EXP budget is larger, and generally unhelpful for the EXP-side methods whose remaining error is not well summarized by a simple EXP residual correction.

\subsection{Exact real benchmark protocol}
\label{app:real_setting}

\paragraph{Dataset and candidate generative models.}
Contexts $X$ are CNN/DailyMail articles \citep{Hermann2015CNN,See2017Pointer}.
The candidate set contains 20 summarization systems: 16 prompt-based abstractive systems built from two Gemma 3 instruction models \citep{Gemma3TechnicalReport}, together with four extractive baselines.
For each context--generative-model pair, the cache stores one summary together with its judged metadata.

\paragraph{Judged rewards and proxy features.}
For each cached output, the judge returns rubric-style signals that are combined into the scalar reward used throughout the real benchmark comparisons.
The real benchmark instead uses a heuristic auxiliary-feature channel derived from cached trajectories and judged-output metadata.
That channel is best read as an OBS proxy, distinct from the hidden representation $\phi^\star$ used in the semi-synthetic study.

\paragraph{Fixed benchmark protocol.}
The real-data study fixes a 48-context subset of relatively difficult CNN/DailyMail articles, selected once by an article-level difficulty score.
Within this fixed pool, each seed performs an $80/20$ context-level train/test split, uses no separate context-level validation split, and tunes hyperparameters only on a held-out fraction of EXP when the estimator requires it.
OBS and EXP are sampled only from the training contexts.
The observational-choice protocol is a fixed paired-sampling design under a self-selection mechanism that routes each decision through a five-system candidate slate before the final model choice.

\paragraph{Budget views and aggregation.}
The real summarization study uses three budget views.
The main-text regret comparisons use
\[
\beta \in \{0, 0.2, 0.5, 0.8, 0.9, 0.99\},\qquad
n_{\obs}\in\{2{,}000,20{,}000\},\qquad
n_{\expd}\in\{20,200\},
\]
with matched smooth and sharpened reward maps, giving the Section~\ref{sec:real} summaries over six \(\beta\) values, two OBS budgets, two EXP budgets, and two reward maps. The DM/DR value diagnostics in Table~\ref{tab:real_dr_breakdown} instead fix \(n_{\obs}=20{,}000\) and use \(n_{\expd}\in\{100,200,800\}\). The winner-map table below is a separate smooth-reward \((\beta,n_{\obs},n_{\expd})\) comparison with
\[
\beta \in \{0, 0.2, 0.5, 0.8, 0.9, 0.99\},\qquad
n_{\obs}\in\{2{,}000,20{,}000\},\qquad
n_{\expd}\in\{20,100\},
\]
reported to show the full \(24\)-cell winner map. All reported means and standard errors are recomputed from seed-level paired runs, with 30 seeds in every reported $(\beta,n_{\obs},n_{\expd},\text{method})$ cell.

\subsection{Tuning on EXP: model-level cross-validation, EXP holdout, and defaults}
\label{app:agent_cv}

We use three EXP-side tuning strategies.
Model-level EXP cross-validation partitions the generative models \emph{appearing in the sampled EXP data for that seed/cell} and validates at the held-out model level over that sampled EXP model set.
EXP holdout reserves a validation subset from the sampled $n_{\expd}$ budget itself and tunes on that holdout.
Fixed defaults use dataset-configured settings with no extra EXP hyperparameter sweep.

More precisely, $n_{\expd}$ is the total EXP budget for a given seed/cell.
If a method uses an explicit EXP holdout, that holdout is carved from the sampled EXP budget itself; no extra EXP rows are sampled outside budget, and no separate context-level validation split is required for this.

For \Grounded{}-, \CVCI{}-, and \CVCIRes{}-style families tuned by model-level EXP cross-validation, let $\mathcal{A}_{\expd}$ be the set of generative models appearing in the sampled EXP data and partition it into $K_{\mathrm{cv}}$ folds $\{\mathcal{A}^{(k)}\}_{k=1}^{K_{\mathrm{cv}}}$.
The corresponding validation indices are
\[
\mathcal{I}^{(k)}_{\mathrm{val}}=\{j\in\expd: A_j\in\mathcal{A}^{(k)}\},\qquad \mathcal{I}^{(k)}_{\mathrm{tr}}=\expd\setminus\mathcal{I}^{(k)}_{\mathrm{val}}.
\]
For a candidate tuning parameter $\eta$ (for example, $\eta=\alpha_{\mathrm{corr}}$ for \Grounded{} or $\eta=\lambda$ for \CVCI{}), fit the corresponding estimator using all OBS data together with the EXP indices in $\mathcal{I}^{(k)}_{\mathrm{tr}}$.
Then evaluate held-out generative-model means with
\begin{equation}
\mathcal{L}^{(k)}(\eta)
=
\sum_{a\in\mathcal{A}^{(k)}} \omega_{k,a}
\Big(\overline{\widehat r}_{\eta}^{(k)}(a)-\overline Y^{(k)}(a)\Big)^2,
\label{eq:cvci_agent_loss}
\end{equation}
where $\overline{\widehat r}_{\eta}^{(k)}(a)=\tfrac{1}{n_{k,a}}\sum_{j\in\mathcal{I}^{(k)}_{\mathrm{val}}:\,A_j=a} \widehat r_{\eta}^{(-k)}(z_j)$, $\overline Y^{(k)}(a)=\tfrac{1}{n_{k,a}}\sum_{j\in\mathcal{I}^{(k)}_{\mathrm{val}}:\,A_j=a} Y_j$, and $n_{k,a}$ is the number of held-out EXP examples for generative model $a$ in fold $k$.
We use count weights $\omega_{k,a}\propto n_{k,a}$, normalized so that $\sum_{a\in\mathcal{A}^{(k)}}\omega_{k,a}=1$, and select the parameter that minimizes the average fold loss.
If too few generative models appear in the sampled EXP data to sustain the requested number of model-level folds, the procedure falls back to sample-level cross-validation and records the effective mode.

In the real main sweep, \Grounded{}, \CVCI{}, and \CVCIRes{} use model-level EXP cross-validation.
The simpler historical linear grounded baseline instead uses an EXP holdout through correction-$\alpha$ tuning on the held-out EXP subset.
\EXPOnly{}, \OBSOnly{}, and \ProxyEXP{} use fixed defaults.

When candidate hyperparameters are tied within numerical tolerance, the implementation resolves ties toward the more regularized / less aggressive solution.
For residual \CVCIRes{} this means preferring larger residual ridge penalties and, if still tied, larger OBS pooling weights.
For holdout-tuned grounded corrections, ties in the correction-$\alpha$ grid are resolved toward smaller shrinkage factors.

Because the \texttt{news\_hard top-48} real benchmark contains only 48 contexts, each $80/20$ context split leaves roughly 10 held-out test contexts per seed.
Single-seed regret is therefore noisy, and all real-benchmark conclusions are aggregated across many seeds.

\subsection{Semi-synthetic controlled validation}
\label{app:synth_validation}
\label{app:synth_details}

The semi-synthetic validation keeps the summarization task family but replaces judged rewards by a known latent generator, so held-out regret can be evaluated directly against the true interventional surface. Each context--mediator pair is mapped to an unobserved representation
\[
\phi^\star(x,m)=H\varphi(x,m)\in\R^{d_\phi},
\]
where \(H\) is a fixed signed-hash projection of sparse hashed text features \(\varphi(X,M)\in\R^{2^{18}}\). The benchmark-specific outcome law is
\[
P_Y(\,\cdot\mid x,m,u)
=
\mathcal L\!\left(
\operatorname{sigmoid}\!\big((w^\star)^\top \phi^\star(x,m)+\lambda^\top u+\varepsilon\big)
\right),
\qquad \varepsilon\sim \mathcal N(0,\sigma_Y^2),
\]
so
\[
r^\star(x,m)
=
\E[Y\mid X=x,M=m,\expd]
=
\E_{U\sim P(\cdot\mid x),\,\varepsilon}\!\Big[
\operatorname{sigmoid}\!\big((w^\star)^\top \phi^\star(x,m)+\lambda^\top U+\varepsilon\big)
\Big].
\]
OBS and EXP share this same \(P_Y\), so the only confounding channel is model choice through \(p_{\obs}^{\beta}(a\mid x,u)\), while EXP randomizes uniformly over the candidate agent set. This appendix validation varies only \(\beta\), \(n_{\obs}\), and \(n_{\expd}\), so changes in the regret winner can be attributed to confounding strength and data budget.

\paragraph{Implementation details.}
The semi-synthetic generator uses sparse hashed text features $\varphi(X,M)\in\R^{2^{18}}$ for all reward models.
The hidden representation has dimension $d_\phi=8$ and is constructed by a seeded signed-hash projection; it is fixed across replicates and is not observed by any estimator.
When hybrid estimators use the proxy map $\psi(z)$ from Section~\ref{sec:estimators}, the PCA dimension is $d_\psi=20$.
The context-only embedding $\psi_X(X)$ that appears in the observational choice model is a separate object.
It is fit on raw article text alone, with no context--output pairs.
Operationally, we sample 5{,}000 contexts from the observational split, fit a TF--IDF vectorizer with English stop-word removal, unigram--bigram features, and a vocabulary cap of 50{,}000 terms, and then apply rank-50 truncated SVD.
Thus $\psi_X(X)\in\R^{50}$ is a low-dimensional context representation used only for model choice and for the $U\mid X$ model in the semi-synthetic generator.
This is distinct from the proxy representation $\psi(X,M)$ in Section~\ref{sec:estimators}, which is learned from OBS auxiliary labels and is used by the hybrid reward estimators.
Reference values and held-out regret metrics are computed on $n_{\mathrm{true}}=100$ and $n_{\mathrm{eval}}=100$ held-out contexts respectively, using $B=5$ simulator draws per context--agent pair.

\paragraph{Fixed protocol in the appendix controlled validation.}
The appendix semi-synthetic validation fixes one confounding protocol and varies only the self-selection weight $\beta$, the observational budget $n_{\obs}$, and the experimental budget $n_{\expd}$.
EXP randomizes uniformly over the candidate agent set.
OBS mixes a context-only softmax policy with a latent-user shortlist policy, where $\beta$ is the probability of drawing from the confounded latent-user component.
This validation uses
\[
\beta \in \{0, 0.2, 0.5, 0.8, 0.9, 0.99\},\qquad
n_{\obs}\in\{2{,}000,20{,}000\},\qquad
n_{\expd}\in\{20,100\},
\]
and reports 30 paired seeds in every available cell.
The direct latent effect on the reward score and the auxiliary-label noise level are both fixed throughout this validation grid.

\paragraph{Budget-axis repetition for pure baselines.}
\OBSOnly{} ignores EXP labels, so its reported value is constant across the two EXP columns for a fixed $(\beta,n_{\obs})$ pair.
\EXPOnly{} ignores OBS outcomes, so its reported value is constant across the two OBS rows for a fixed $(\beta,n_{\expd})$ pair.
This repetition is carried through consistently in the synthetic regime table and summary table.

\subsection{Appendix results for the summarization benchmark}
\label{app:sweep_details}

\paragraph{Synthetic summarization results.}
\label{app:synth_regime_details}
Table~\ref{tab:synth_beta_winner_map} shows the regret winner in each synthetic summarization setting over six \(\beta\) values and four \((n_{\obs},n_{\expd})\) budget settings.

\begin{table}[tb]
\centering
\scriptsize
\setlength{\tabcolsep}{4pt}
\renewcommand{\arraystretch}{1.15}
\resizebox{0.98\linewidth}{!}{%
\begin{tabular}{lcccccc}
\toprule
Budget cell & $\beta=0$ & $\beta=0.2$ & $\beta=0.5$ & $\beta=0.8$ & $\beta=0.9$ & $\beta=0.99$ \\
\midrule
\shortstack{$n_{\obs}=2{,}000$\\$n_{\expd}=20$}
& \wincell{cvcirescolor}{\CVCIRes{}}{0.0004}
& \wincell{cvcicolor}{\CVCI{}}{0.0002}
& \wincell{cvcicolor}{\CVCI{}}{0.0009}
& \wincell{obsonlycolor}{\OBSOnly{}}{0.0004}
& \wincell{groundedcolor}{\Grounded{}}{0.0010}
& \wincell{groundedcolor}{\Grounded{}}{0.0004} \\
\shortstack{$n_{\obs}=20{,}000$\\$n_{\expd}=20$}
& \wincell{cvcicolor}{\CVCI{}}{0.0005}
& \wincell{cvcicolor}{\CVCI{}}{0.0015}
& \wincell{cvcicolor}{\CVCI{}}{0.0015}
& \wincell{cvcicolor}{\CVCI{}}{0.0021}
& \wincell{cvcicolor}{\CVCI{}}{0.0018}
& \wincell{cvcicolor}{\CVCI{}}{0.0017} \\
\shortstack{$n_{\obs}=2{,}000$\\$n_{\expd}=100$}
& \wincell{cvcicolor}{\CVCI{}}{0.0008}
& \wincell{cvcicolor}{\CVCI{}}{0.0017}
& \wincell{cvcicolor}{\CVCI{}}{0.0012}
& \wincell{obsonlycolor}{\OBSOnly{}}{0.0006}
& \wincell{obsonlycolor}{\OBSOnly{}}{0.0008}
& \wincell{obsonlycolor}{\OBSOnly{}}{0.0005} \\
\shortstack{$n_{\obs}=20{,}000$\\$n_{\expd}=100$}
& \wincell{cvcicolor}{\CVCI{}}{0.0002}
& \wincell{exponlycolor}{\EXPOnly{}}{0.0012}
& \wincell{exponlycolor}{\EXPOnly{}}{0.0016}
& \wincell{exponlycolor}{\EXPOnly{}}{0.0013}
& \wincell{exponlycolor}{\EXPOnly{}}{0.0012}
& \wincell{exponlycolor}{\EXPOnly{}}{0.0016} \\
\bottomrule
\end{tabular}%
}
\caption{Regret-optimal estimator on the semi-synthetic summarization benchmark. The number beneath each winner is the regret gap to the runner-up.}
\label{tab:synth_beta_winner_map}
\end{table}

The smallest-budget row \((2{,}000,20)\) is the most mixed: \CVCIRes{} wins once, \CVCI{} wins twice, \OBSOnly{} wins once, and \Grounded{} wins twice. At \((20{,}000,20)\), \CVCI{} wins all six beta values. At \((2{,}000,100)\), the winner shifts from \CVCI{} at lower \(\beta\) to \OBSOnly{} at higher \(\beta\). At \((20{,}000,100)\), \EXPOnly{} wins all but the \(\beta=0\) cell.

Across the 24 settings, \CVCI{} is the most stable overall winner, but no single method wins everywhere. Most regret gaps are small, so Table~\ref{tab:synth_beta_winner_map} is best read as a map of where the winner changes.

\paragraph{Additional synthetic aggregate interpretation.}
Across the same 24 settings, \CVCI{} is also the aggregate regret leader. \EXPOnly{} and \OBSOnly{} form the next tier, while \Grounded{} and \CVCIRes{} win only a few cells. The ranking diagnostics are flatter than regret: no single method leads top-1 or top-3 across the grid.

\paragraph{Real summarization results on the smooth auxiliary comparison.}
\label{app:real_regime_details}
Table~\ref{tab:real_beta_winner_map} shows the regret winner in each real summarization setting for the auxiliary smooth-reward comparison over six \(\beta\) values and four \((n_{\obs},n_{\expd})\) budget settings.

\begin{table}[tb]
\centering
\scriptsize
\setlength{\tabcolsep}{4pt}
\renewcommand{\arraystretch}{1.15}
\resizebox{0.98\linewidth}{!}{%
\begin{tabular}{lcccccc}
\toprule
Budget cell & $\beta=0$ & $\beta=0.2$ & $\beta=0.5$ & $\beta=0.8$ & $\beta=0.9$ & $\beta=0.99$ \\
\midrule
\shortstack{$n_{\obs}=2{,}000$\\$n_{\expd}=20$}
& \wincell{groundedcolor}{\Grounded{}}{0.0002}
& \wincell{groundedcolor}{\Grounded{}}{0.0011}
& \wincell{groundedcolor}{\Grounded{}}{0.0036}
& \wincell{obsonlycolor}{\OBSOnly{}}{0.0000}
& \wincell{obsonlycolor}{\OBSOnly{}}{0.0008}
& \wincell{groundedcolor}{\Grounded{}}{0.0006} \\
\shortstack{$n_{\obs}=20{,}000$\\$n_{\expd}=20$}
& \wincell{groundedcolor}{\Grounded{}}{0.0016}
& \wincell{groundedcolor}{\Grounded{}}{0.0049}
& \wincell{exponlycolor}{\EXPOnly{}}{0.0013}
& \wincell{exponlycolor}{\EXPOnly{}}{0.0062}
& \wincell{exponlycolor}{\EXPOnly{}}{0.0019}
& \wincell{exponlycolor}{\EXPOnly{}}{0.0019} \\
\shortstack{$n_{\obs}=2{,}000$\\$n_{\expd}=100$}
& \wincell{obsonlycolor}{\OBSOnly{}}{0.0007}
& \wincell{obsonlycolor}{\OBSOnly{}}{0.0002}
& \wincell{obsonlycolor}{\OBSOnly{}}{0.0002}
& \wincell{cvcicolor}{\CVCI{}}{0.0000}
& \wincell{cvcicolor}{\CVCI{}}{0.0000}
& \wincell{cvcicolor}{\CVCI{}}{0.0008} \\
\shortstack{$n_{\obs}=20{,}000$\\$n_{\expd}=100$}
& \wincell{cvcirescolor}{\CVCIRes{}}{0.0048}
& \wincell{cvcicolor}{\CVCI{}}{0.0009}
& \wincell{cvcicolor}{\CVCI{}}{0.0052}
& \wincell{cvcicolor}{\CVCI{}}{0.0012}
& \wincell{cvcicolor}{\CVCI{}}{0.0013}
& \wincell{cvcicolor}{\CVCI{}}{0.0046} \\
\bottomrule
\end{tabular}%
}
\caption{Regret-optimal estimator on the real summarization benchmark. The number beneath each winner is the regret gap to the runner-up.}
\label{tab:real_beta_winner_map}
\end{table}

At \((2{,}000,20)\), \Grounded{} wins four beta values and \OBSOnly{} wins two; the \(\beta=0.8\) cell is essentially a tie. At \((20{,}000,20)\), \Grounded{} wins at low \(\beta\), while \EXPOnly{} wins at higher \(\beta\). Once \(n_{\expd}=100\), the high-OBS row \((20{,}000,100)\) is dominated by \CVCI{}, with a single \CVCIRes{} win at \(\beta=0\), while the low-OBS row \((2{,}000,100)\) shifts from \OBSOnly{} at lower \(\beta\) to \CVCI{} at higher \(\beta\).

\paragraph{Additional real ranking--regret comparison.}
Across these 24 settings, \CVCI{} is the aggregate regret leader and \Grounded{} is second. Ranking tells a different story: \Grounded{} wins top-1 and top-3 in all 24 settings but wins regret in only 6.

\subsection{Public cached coding benchmark}
\label{app:coding_geometry}

\paragraph{Fixed protocol.}
The coding benchmark uses a public cached grid of issue descriptions, patch mediators, and execution-derived rewards.
The coding comparison in Section~\ref{sec:real} fixes
\[
\beta=0.5,\qquad n_{\obs}=2{,}000,\qquad n_{\expd}=100,
\]
uses outcome-precision \(64\), and evaluates only the feature-based implementations of the estimator families.
It then varies two reward-design parameters:
\[
\text{fix-success weight}\in\{0,0.25,0.5,0.75,1\},
\qquad
\text{weakest-link weight}\in\{0,0.25,0.5,0.75,1\},
\]
with 10 paired seeds in every cell.
The first parameter rescales the fix-success component of the reward, while the second moves the non-fix patch-quality score from additive aggregation toward a weakest-link rule.

\paragraph{Appendix coding results.}
In this coding appendix comparison, we treat a winner change as substantive when the top-two regret gap exceeds \(3\times 10^{-3}\); smaller gaps are near-ties. Table~\ref{tab:coding_geometry_winner_map} reports the regret winner in each cell. Every top-two gap is below \(0.01\), and \(20\) of the \(25\) cells are below \(1.5\times 10^{-3}\). Most flips between \OBSStruct{} and \CVCIStruct{} therefore reflect near-ties in the high-fix-success part of the grid. The clearest differences appear at low fix-success weight under additive scoring, where \RepStruct{} outperforms the OBS-based methods, and in the steady narrowing of the grounded-versus-OBS gap as weakest-link scoring becomes stronger.

\begin{table}[tb]
\centering
\scriptsize
\setlength{\tabcolsep}{4pt}
\renewcommand{\arraystretch}{1.12}
\resizebox{0.98\linewidth}{!}{%
\begin{tabular}{lccccc}
\toprule
Fix-success weight\textbackslash{}weakest-link weight & \(0\) & \(0.25\) & \(0.5\) & \(0.75\) & \(1.0\) \\
\midrule
\(0.00\)
& \wincell{proxyexpcolor}{\RepStruct{}}{0.0039}
& \wincell{proxyexpcolor}{\RepStruct{}}{0.0033}
& \wincell{cvcirescolor}{\CVCIResStruct{}}{0.0001}
& \wincell{obsonlycolor}{\OBSStruct{}}{0.0001}
& \wincell{obsonlycolor}{\OBSStruct{}}{0.0002} \\
\(0.25\)
& \wincell{proxyexpcolor}{\RepStruct{}}{0.0006}
& \wincell{obsonlycolor}{\OBSStruct{}}{0.0004}
& \wincell{cvcicolor}{\CVCIStruct{}}{0.0004}
& \wincell{cvcicolor}{\CVCIStruct{}}{0.0013}
& \wincell{cvcicolor}{\CVCIStruct{}}{0.0009} \\
\(0.50\)
& \wincell{obsonlycolor}{\OBSStruct{}}{0.0004}
& \wincell{obsonlycolor}{\OBSStruct{}}{0.0010}
& \wincell{cvcicolor}{\CVCIStruct{}}{0.0003}
& \wincell{cvcicolor}{\CVCIStruct{}}{0.0008}
& \wincell{cvcicolor}{\CVCIStruct{}}{0.0004} \\
\(0.75\)
& \wincell{cvcicolor}{\CVCIStruct{}}{0.0013}
& \wincell{obsonlycolor}{\OBSStruct{}}{0.0029}
& \wincell{obsonlycolor}{\OBSStruct{}}{0.0004}
& \wincell{obsonlycolor}{\OBSStruct{}}{0.0003}
& \wincell{obsonlycolor}{\OBSStruct{}}{0.0013} \\
\(1.00\)
& \wincell{obsonlycolor}{\OBSStruct{}}{0.0011}
& \wincell{cvcicolor}{\CVCIStruct{}}{0.0007}
& \wincell{cvcicolor}{\CVCIStruct{}}{0.0001}
& \wincell{obsonlycolor}{\OBSStruct{}}{0.0068}
& \wincell{obsonlycolor}{\OBSStruct{}}{0.0056} \\
\bottomrule
\end{tabular}%
}
\caption{Regret-optimal estimator on the coding benchmark at fixed \(n_{\obs}=2{,}000\), \(n_{\expd}=100\), \(\beta=0.5\), and reward-noise precision \(64\). Columns give the weakest-link weight in the non-fix patch-quality score, and rows give the weight on true fix success. Most top-two gaps are small, so the table mainly shows where the regret winner changes.}
\label{tab:coding_geometry_winner_map}
\end{table}

\paragraph{Secondary axes omitted from the main text.}
Additional coding comparisons over confounding strength and outcome noise were qualitatively weaker.
They changed absolute difficulty, but they did not produce ranking shifts as clear as those induced by the fix-success weight and weakest-link weight in this fixed-budget probe.
For that reason, the main text centers this reward-design comparison and leaves the other axes to this appendix discussion.

\subsection{Implementation notes}
The synthetic validation uses only cached data generation; no external API calls occur after the generator pool is constructed.
\DR{} uses EXP propensities from uniform randomization together with cross-fitting for reward models trained on EXP.
Additional synthetic diagnostics and older experimental variants are omitted from the present main-text comparisons.

\section{Rubric subscores for CNN/DailyMail summaries}
\label{app:rubric}

Given an article $X$ and candidate summary $M$, the judge outputs four integer subscores
$\tilde s(X,M)\in\{0,1,2,3,4,5\}^4$:
\emph{faithfulness}, \emph{coverage}, \emph{clarity}, and \emph{conciseness}.
Scores use the anchors below; intermediate values interpolate between adjacent anchors.

\paragraph{Faithfulness (0--5).}
Factual consistency with the article. Unsupported claims include invented entities, numbers, events, causal attributions, or quotes.
\begin{itemize}\itemsep0pt \parskip0pt \parsep0pt
\item \textbf{5:} No unsupported claims; all salient statements are grounded in $X$.
\item \textbf{4:} Minor unsupported detail(s) that do not affect the main facts.
\item \textbf{3:} Multiple unsupported or weakly supported statements; main story remains recognizable.
\item \textbf{2:} Several unsupported claims affecting key details (e.g., actor/action/outcome/number).
\item \textbf{1:} Many unsupported claims; substantial contradiction or fabrication.
\item \textbf{0:} Predominantly hallucinatory or contradictory relative to $X$.
\end{itemize}

\paragraph{Coverage (0--5).}
Inclusion of the article's major points (headline facts, primary actors, core events, outcomes, and key qualifiers).
\begin{itemize}\itemsep0pt \parskip0pt \parsep0pt
\item \textbf{5:} Covers all major points; omissions are limited to minor details.
\item \textbf{4:} Covers most major points; at most one major point missing.
\item \textbf{3:} Covers some major points; several major omissions.
\item \textbf{2:} Covers few major points; summary reflects only a small slice of the article.
\item \textbf{1:} Minimal coverage; largely misses the article's main content.
\item \textbf{0:} Misses most major points or is largely off-topic relative to $X$.
\end{itemize}

\paragraph{Clarity (0--5).}
Readability and organization (coherence, grammaticality, referential clarity, and logical flow).
\begin{itemize}\itemsep0pt \parskip0pt \parsep0pt
\item \textbf{5:} Clear and well-structured; unambiguous references and fluent phrasing.
\item \textbf{4:} Generally clear; minor awkwardness or minor ambiguity.
\item \textbf{3:} Understandable with effort; noticeable disfluency, repetition, or unclear references.
\item \textbf{2:} Hard to follow; frequent grammatical issues or unclear structure.
\item \textbf{1:} Mostly unclear; major coherence failures.
\item \textbf{0:} Unreadable or incoherent.
\end{itemize}

\paragraph{Conciseness (0--5).}
Information density at an appropriate length; redundancy and irrelevant detail reduce the score; excessive brevity that omits necessary content also reduces the score.
\begin{itemize}\itemsep0pt \parskip0pt \parsep0pt
\item \textbf{5:} Efficient summary with minimal redundancy and no filler.
\item \textbf{4:} Slight redundancy or mild over/under-length without major impact.
\item \textbf{3:} Noticeable redundancy or length mismatch; still usable as a summary.
\item \textbf{2:} Clearly too verbose or too terse; substantial inefficiency or truncation.
\item \textbf{1:} Extremely verbose or extremely short; poor summary form.
\item \textbf{0:} Pathological length or pervasive redundancy; not a usable summary.
\end{itemize}

%% file: sections/B_theory_proofs_revised.tex
\section{Full statements and proofs for identification and population estimator comparisons}
\label{app:theory}

This appendix collects the proof of the main identification result and the full statements and proofs for the population estimator comparisons referenced in Section~\ref{sec:estimators}. It also explains how these approximation-class results motivate the empirical target-alignment diagnostics used in the benchmarks; those diagnostics are not additional formal results.
Throughout, $Z=(X,M)$, $P_E$ denotes the EXP law of $Z$, and for a fixed target outcome $Y$ we write
\begin{equation}
\mathcal R_E(f)=\sigma_E^2+\|f-r^\star\|_E^2,
\qquad
\|g\|_E^2:=\E_{Z\sim P_E}[g(Z)^2],
\qquad
\sigma_E^2:=\E\!\left[(Y-r^\star(Z))^2\mid \expd\right],
\label{eq:exp_risk_decomposition}
\end{equation}
where $r^\star(z):=\E[Y\mid Z=z,\expd]$.
For readability, we ignore the final clipping to $[0,1]$ and absorb the intercept into both $\psi$ and $\varphi$ by augmenting them with a constant feature.

\subsection{Proof of Theorem~\ref{thm:replay_identification}}
\label{app:proof_replay_identification}

\begin{proof}[Proof of Theorem~\ref{thm:replay_identification}]
By Assumption A1, $A\perp U\mid X$ in EXP. Assumption A4 states the structural restriction $U\perp M\mid X,A$ in OBS, in EXP, and under $\doop(A=a)$, so that conditional on $(X,A)$ the mediator $M$ carries no residual information about $U$. Combining A1 and A4 with the chain rule for conditional independence gives
\[
U \perp (A,M)\mid X
\qquad\text{in EXP},
\]
so on the shared EXP/SIM support the scoring rule $r^\star(x,m)=\E[Y\mid X=x,M=m,\expd]$ is identified from EXP and is free of self-selection bias. Assumption A2 gives $p_{\simu}(m\mid x,a)=\Pp(M=m\mid X=x,\doop(A=a))$, so for any $x$ in the EXP support,
\begin{align*}
q(x,a)
&= \E\!\left[Y(\doop(A=a))\mid X=x\right] \\
&= \E_{M\sim p_{\simu}(\cdot\mid x,a)}\!\left[\E[Y\mid X=x,M=M,\doop(A=a)]\right] \\
&= \E_{M\sim p_{\simu}(\cdot\mid x,a)}\!\left[r^\star(x,M)\right],
\end{align*}
where the last equality uses Assumption A3 (the conditional outcome law given $(X,M,U)$ is the same in OBS, EXP, and under $\doop(A=a)$, and $r^\star$ integrates $U$ out under the EXP $U\mid X$ distribution, which by A1 equals the post-intervention $U\mid X$ distribution). This proves \eqref{eq:ident_q}. Averaging $q(X,a)$ under the EXP context distribution and invoking Assumption A5 ($P_X^{\mathrm{tgt}}=P_X^{\expd}$) gives \eqref{eq:ident_mu}. Outside Assumption A5, only $q(x,a)$ on the shared support is identified; reweighting to a different target context distribution requires a standard covariate-shift adjustment.
\end{proof}

Define the linear classes
\[
\mathcal H_\psi:=\{z\mapsto \theta^\top\psi(z):\theta\in\R^{d_\psi}\},
\qquad
\mathcal F_\varphi:=\{z\mapsto w^\top\varphi(z):w\in\R^{d_\varphi}\},
\]
and the grounded / residual affine class
\[
\mathcal G(f_{\obs},\psi):=\{f_{\obs}-h:\ h\in\mathcal H_\psi\}=f_{\obs}+\mathcal H_\psi,
\]
where the equality uses that $\mathcal H_\psi$ is a linear space.

\subsection{Summary of the fixed-target comparisons and empirical alignment diagnostics}

The main text uses the population comparisons as conditional guidance for reading the experiments, not as a guarantee that one estimator dominates. We separate the formal results included below from empirical diagnostics that are motivated by the same approximation-gap viewpoint.

\paragraph{Formal results included below.}

\begin{itemize}[leftmargin=*]
\item \textbf{At the oracle level, correcting OBS cannot hurt.}
Projecting OBS bias onto the correction space weakly improves EXP risk relative to \OBSOnly{}; shrinkage matters only because the estimated correction can be noisy in finite samples.
See Theorem~\ref{thm:grounded_vs_obs}.

\item \textbf{\Grounded{} expands the proxy-side approximation class.}
If the OBS baseline is not already linear in $\psi$, then correcting an expressive baseline by a low-dimensional proxy residual can represent targets that a pure $\psi$-linear EXP fit cannot.
See Theorem~\ref{thm:grounded_function_class}.

\item \textbf{\CVCIRes{} is favored when the hard part is already in the OBS baseline.}
Residualization beats direct pooling when $f_{\obs}$ absorbs the difficult part of the reward surface and the remaining discrepancy is simpler in proxy space than the full target is in raw text features.
See Theorem~\ref{thm:residual_vs_pooling}.

\item \textbf{In the shared linear proxy special case, \Grounded{} is centered ridge.}
When both the OBS baseline and the EXP target are linear in the same proxy representation, \Grounded{} becomes ridge centered at the OBS coefficient, and it beats \EXPOnly{} exactly when that OBS center is a better shrinkage target than zero.
See Theorem~\ref{thm:grounded_centered_ridge} and Corollary~\ref{cor:grounded_vs_exp_linear}.
\end{itemize}

\paragraph{Empirical diagnostics motivated by the formal results.}
The experiments do not test the oracle conditions directly. They report proxy-alignment and reward-shape diagnostics that are related to the approximation classes in the formal results. In summarization, the smooth-to-sharpened reward comparison changes the target while keeping the cached outputs fixed. In coding, the fix-success sweep moves the target away from auxiliary patch-quality features and toward true fix success. These diagnostics are empirical probes of estimator--target alignment, not additional theorems.

\begin{itemize}[leftmargin=*]
\item \textbf{Target--proxy alignment is measured empirically.}
The representation-based estimators are most useful when the target reward varies along directions retained by the auxiliary proxy. The summarization and coding experiments report held-out target-fit \(R^2\) as a proxy for this alignment.

\item \textbf{Reward definitions can change estimator rankings through approximation-class alignment.}
Changing the reward definition can move the target toward or away from the approximation class used by an estimator. The smooth/sharpened summarization comparison and the coding fix-success sweep are diagnostics consistent with this mechanism.

\item \textbf{Regret is the primary downstream metric.}
The experiments report recommendation regret because the evaluation objective is action selection; RMSE and ranking metrics are secondary diagnostics for reward-surface fit and model-level value fit.
\end{itemize}

\input{sections/B_theory_proofs}

%% file: sections/B_theory_proofs.tex
\subsection{What representation correction buys}

\begin{theorem}[What representation correction buys]
\label{thm:grounded_function_class}
For any fixed baseline $f_{\obs}$, the grounded class $\mathcal G(f_{\obs},\psi)$ satisfies the following three statements.
\begin{enumerate}[label=(\alph*)]
\item If there exists $\theta_\star\in\R^{d_\psi}$ such that
\[
r^\star(z)=f_{\obs}(z)-\theta_\star^\top\psi(z)
\qquad\text{for all }z,
\]
then $r^\star\in\mathcal G(f_{\obs},\psi)$.
Conditional on the fitted baseline $f_{\obs}$, EXP only needs to estimate the $d_\psi$-dimensional residual coefficient $\theta_\star$.

\item If $f_{\obs}\notin \mathcal H_\psi$, then
\[
\mathcal G(f_{\obs},\psi)\neq \mathcal H_\psi,
\qquad
f_{\obs}\in \mathcal G(f_{\obs},\psi)\setminus\mathcal H_\psi.
\]
So grounded correction can represent targets that a pure $\psi$-linear EXP fit cannot.

\item If $f_{\obs}\in \mathcal H_\psi$, then
\[
\mathcal G(f_{\obs},\psi)= \mathcal H_\psi.
\]
In that special case, representation correction does not enlarge the function class.
\end{enumerate}
\end{theorem}

\begin{proof}
We prove the three claims in order.

For part (a), the displayed identity
\[
r^\star(z)=f_{\obs}(z)-\theta_\star^\top\psi(z)
\]
is exactly the statement that $r^\star\in\mathcal G(f_{\obs},\psi)$.
Once $f_{\obs}$ is treated as fixed, the only EXP-side unknown is the $d_\psi$-dimensional coefficient vector $\theta_\star$.

For part (b), note that $0\in\mathcal H_\psi$, hence
\[
f_{\obs}=f_{\obs}-0\in\mathcal G(f_{\obs},\psi).
\]
If $\mathcal G(f_{\obs},\psi)=\mathcal H_\psi$, then this would imply $f_{\obs}\in\mathcal H_\psi$, contradicting the assumption.
Therefore $\mathcal G(f_{\obs},\psi)\neq\mathcal H_\psi$, and $f_{\obs}$ is a member of $\mathcal G(f_{\obs},\psi)$ that does not belong to $\mathcal H_\psi$.

For part (c), suppose $f_{\obs}(z)=\beta_{\obs}^\top\psi(z)$ for some $\beta_{\obs}\in\R^{d_\psi}$.
Every $g\in\mathcal G(f_{\obs},\psi)$ has the form
\[
g(z)=\beta_{\obs}^\top\psi(z)-\theta^\top\psi(z)
=(\beta_{\obs}-\theta)^\top\psi(z)\in\mathcal H_\psi,
\]
so $\mathcal G(f_{\obs},\psi)\subseteq\mathcal H_\psi$.
Conversely, for any $g(z)=\beta^\top\psi(z)\in\mathcal H_\psi$, choosing $\theta=\beta_{\obs}-\beta$ gives
\[
g(z)=f_{\obs}(z)-\theta^\top\psi(z)\in\mathcal G(f_{\obs},\psi).
\]
Thus $\mathcal H_\psi\subseteq\mathcal G(f_{\obs},\psi)$, proving equality.
\end{proof}

\subsection{Residualization versus direct pooling}

\begin{theorem}[Residualization versus direct pooling]
\label{thm:residual_vs_pooling}
Let $f_{\obs}$ be the observational baseline used by \Grounded{} and \CVCIRes{}, and assume $f_{\obs},r^\star\in L^2(P_E)$ together with finite-dimensional linear classes $\mathcal F_\varphi$ and $\mathcal H_\psi$, so that both classes are closed in $L^2(P_E)$ and the projections in question exist (when $\arg\min$ is non-unique, any minimizer suffices, since the displayed risks depend only on the projections).
Define the oracle direct and residual predictors by
\[
f^\dagger_{\mathrm{cvci}}
\in
\arg\min_{f\in\mathcal F_\varphi}\mathcal R_E(f),
\qquad
h^\dagger_{\mathrm{res}}
\in
\arg\min_{h\in\mathcal H_\psi}\mathcal R_E(f_{\obs}+h).
\]
Then
\[
\mathcal R_E(f_{\obs}+h^\dagger_{\mathrm{res}})
\le
\mathcal R_E(f^\dagger_{\mathrm{cvci}})
\]
if and only if
\[
\operatorname{dist}_E(r^\star,f_{\obs}+\mathcal H_\psi)
\le
\operatorname{dist}_E(r^\star,\mathcal F_\varphi),
\]
where $\operatorname{dist}_E(r,\mathcal A):=\inf_{g\in\mathcal A}\|r-g\|_E$.
In particular, if
\[
r^\star-f_{\obs}\in\mathcal H_\psi
\qquad\text{but}\qquad
r^\star\notin\overline{\mathcal F_\varphi},
\]
then
\[
\mathcal R_E(f_{\obs}+h^\dagger_{\mathrm{res}})
<
\mathcal R_E(f^\dagger_{\mathrm{cvci}}).
\]
\end{theorem}

\begin{proof}
By \eqref{eq:exp_risk_decomposition},
\[
\inf_{f\in\mathcal F_\varphi}\mathcal R_E(f)
=
\sigma_E^2+\inf_{f\in\mathcal F_\varphi}\|f-r^\star\|_E^2
=
\sigma_E^2+\operatorname{dist}_E^2(r^\star,\mathcal F_\varphi).
\]
Likewise,
\[
\inf_{h\in\mathcal H_\psi}\mathcal R_E(f_{\obs}+h)
=
\sigma_E^2+\inf_{h\in\mathcal H_\psi}\|f_{\obs}+h-r^\star\|_E^2
=
\sigma_E^2+\operatorname{dist}_E^2(r^\star,f_{\obs}+\mathcal H_\psi).
\]
Subtracting the common irreducible term $\sigma_E^2$ gives the equivalence.

If $r^\star-f_{\obs}\in\mathcal H_\psi$, then
\[
\operatorname{dist}_E(r^\star,f_{\obs}+\mathcal H_\psi)=0.
\]
If also $r^\star\notin\overline{\mathcal F_\varphi}$, then
\[
\operatorname{dist}_E(r^\star,\mathcal F_\varphi)>0.
\]
Hence
\[
\mathcal R_E(f_{\obs}+h^\dagger_{\mathrm{res}})
<
\mathcal R_E(f^\dagger_{\mathrm{cvci}}).
\qedhere
\]
\end{proof}

\subsection{Oracle grounding and shrinkage}

\begin{theorem}[Oracle grounding is no worse than OBS-only]
\label{thm:grounded_vs_obs}
Let $b:=f_{\obs}-r^\star$ denote the OBS bias function on EXP.
Let $\mathcal H_\psi\subset L^2(P_E)$ be the resulting closed linear correction space, and let
\[
h^\dagger:=\Pi_{\mathcal H_\psi}b
\]
be the $L^2(P_E)$ projection of $b$ onto that space.
For any $\alpha\in[0,1]$, define the oracle grounded predictor
\[
f_{\mathrm{G},\alpha}:=f_{\obs}-\alpha h^\dagger.
\]
Then
\begin{equation}
\mathcal R_E(f_{\mathrm{G},\alpha})
=
\mathcal R_E(f_{\obs})-(2\alpha-\alpha^2)\|h^\dagger\|_E^2.
\label{eq:oracle_grounded_gain}
\end{equation}
Hence
\[
\mathcal R_E(f_{\mathrm{G},\alpha})\le \mathcal R_E(f_{\obs})
\qquad\text{for all }\alpha\in[0,1],
\]
with strict inequality whenever $\alpha>0$ and $h^\dagger\neq 0$.
\end{theorem}

\begin{proof}
Write
\[
b=h^\dagger+u,
\qquad
u:=b-h^\dagger.
\]
Because $h^\dagger$ is the orthogonal projection of $b$ onto the closed linear subspace $\mathcal H_\psi$, we have $u\perp h^\dagger$ in $L^2(P_E)$.

Now
\[
f_{\mathrm{G},\alpha}-r^\star
=
(f_{\obs}-r^\star)-\alpha h^\dagger
=
u+(1-\alpha)h^\dagger.
\]
Therefore,
\[
\mathcal R_E(f_{\mathrm{G},\alpha})-\sigma_E^2
=
\|u+(1-\alpha)h^\dagger\|_E^2
=
\|u\|_E^2+(1-\alpha)^2\|h^\dagger\|_E^2,
\]
where the cross term vanishes by orthogonality.
On the other hand,
\[
\mathcal R_E(f_{\obs})-\sigma_E^2
=
\|b\|_E^2
=
\|u+h^\dagger\|_E^2
=
\|u\|_E^2+\|h^\dagger\|_E^2.
\]
Subtracting the two displays gives
\[
\mathcal R_E(f_{\mathrm{G},\alpha})-\mathcal R_E(f_{\obs})
=
\bigl((1-\alpha)^2-1\bigr)\|h^\dagger\|_E^2
=
-(2\alpha-\alpha^2)\|h^\dagger\|_E^2,
\]
which is \eqref{eq:oracle_grounded_gain}.
Since $2\alpha-\alpha^2\ge 0$ on $[0,1]$, the inequality follows, and it is strict whenever $\alpha>0$ and $h^\dagger\neq 0$.
\end{proof}

\paragraph{Derivation of the noisy-correction identity.}
For any $\widetilde h\in L^2(P_E)$,
\begin{equation}
\mathcal R_E(f_{\obs}-\alpha \widetilde h)-\mathcal R_E(f_{\obs})
=
\alpha^2\|\widetilde h\|_E^2 - 2\alpha\langle b,\widetilde h\rangle_E.
\label{eq:noisy_correction_identity}
\end{equation}
Indeed, $f_{\obs}-\alpha\widetilde h-r^\star=b-\alpha\widetilde h$, so
\[
\mathcal R_E(f_{\obs}-\alpha\widetilde h)-\sigma_E^2
=
\|b-\alpha\widetilde h\|_E^2
=
\|b\|_E^2+\alpha^2\|\widetilde h\|_E^2-2\alpha\langle b,\widetilde h\rangle_E.
\]
Subtracting $\mathcal R_E(f_{\obs})-\sigma_E^2=\|b\|_E^2$ yields \eqref{eq:noisy_correction_identity}.

\subsection{The shared linear proxy special case}

\begin{theorem}[Grounded is centered ridge in the linear proxy special case]
\label{thm:grounded_centered_ridge}
Assume the observational baseline is already linear in the proxy representation,
\[
f_{\obs}(z)=\beta_{\obs}^\top\psi(z),
\]
with the intercept absorbed into $\psi$.
Let the EXP sample be $\{(\psi_j,Y_j)\}_{j=1}^n$, and write
\[
\Psi=
\begin{bmatrix}
\psi_1^\top\\
\vdots\\
\psi_n^\top
\end{bmatrix}\in\R^{n\times d_\psi},
\qquad
Y=(Y_1,\dots,Y_n)^\top.
\]
Under full correction $\alpha_{\mathrm{corr}}=1$, grounded correction solves
\[
\widehat\delta
\in
\arg\min_{\delta\in\R^{d_\psi}}
\|\Psi\beta_{\obs}-Y-\Psi\delta\|_2^2+\lambda\|\delta\|_2^2
\]
and outputs
\[
\widehat\beta_{\mathrm G}:=\beta_{\obs}-\widehat\delta.
\]
Then $\widehat\beta_{\mathrm G}$ is exactly the centered ridge estimator
\[
\widehat\beta_{\mathrm G}
\in
\arg\min_{\beta\in\R^{d_\psi}}
\|Y-\Psi\beta\|_2^2+\lambda\|\beta-\beta_{\obs}\|_2^2.
\]
Equivalently,
\[
\widehat\beta_{\mathrm G}
=
(\Psi^\top\Psi+\lambda I)^{-1}(\Psi^\top Y+\lambda\beta_{\obs}).
\]
For general $\alpha_{\mathrm{corr}}\in[0,1]$,
\[
\widehat\beta_{\mathrm G,\alpha_{\mathrm{corr}}}
:=
\beta_{\obs}-\alpha_{\mathrm{corr}}\widehat\delta
=
(1-\alpha_{\mathrm{corr}})\beta_{\obs}
+
\alpha_{\mathrm{corr}}\widehat\beta_{\mathrm G}.
\]
\end{theorem}

\begin{proof}
Set
\[
\beta=\beta_{\obs}-\delta.
\]
Then $\delta=\beta_{\obs}-\beta$, so
\[
\Psi\beta_{\obs}-Y-\Psi\delta
=
\Psi\beta_{\obs}-Y-\Psi(\beta_{\obs}-\beta)
=
\Psi\beta-Y.
\]
Therefore
\[
\|\Psi\beta_{\obs}-Y-\Psi\delta\|_2^2+\lambda\|\delta\|_2^2
=
\|Y-\Psi\beta\|_2^2+\lambda\|\beta-\beta_{\obs}\|_2^2.
\]
So minimizing over $\delta$ is equivalent to minimizing over $\beta$, which proves the centered-ridge representation.

For the closed form, differentiate the centered-ridge objective:
\[
-2\Psi^\top(Y-\Psi\beta)+2\lambda(\beta-\beta_{\obs})=0,
\]
hence
\[
(\Psi^\top\Psi+\lambda I)\beta=\Psi^\top Y+\lambda\beta_{\obs},
\]
which yields the stated formula for $\widehat\beta_{\mathrm G}$.

Finally,
\[
\widehat\beta_{\mathrm G,\alpha_{\mathrm{corr}}}
=
\beta_{\obs}-\alpha_{\mathrm{corr}}\widehat\delta
=
\beta_{\obs}-\alpha_{\mathrm{corr}}(\beta_{\obs}-\widehat\beta_{\mathrm G})
=
(1-\alpha_{\mathrm{corr}})\beta_{\obs}
+
\alpha_{\mathrm{corr}}\widehat\beta_{\mathrm G}.
\qedhere
\]
\end{proof}

\begin{corollary}[When the OBS center beats EXP-only]
\label{cor:grounded_vs_exp_linear}
Assume in addition that the EXP target is linear in the same representation,
\[
Y=\Psi\beta_\star+\varepsilon,
\qquad
\E[\varepsilon\mid \Psi]=0,
\qquad
\operatorname{Var}(\varepsilon\mid \Psi)=\sigma^2 I_n.
\]
Let $S:=\Psi^\top\Psi$, $A:=(S+\lambda I)^{-1}$, and define the EXP-only ridge estimator
\[
\widehat\beta_{\expd}:=A\Psi^\top Y.
\]
For any positive semidefinite matrix $G$, write $\|u\|_G^2:=u^\top G u$.
Then, conditional on $\Psi$,
\[
\E\!\left[\|\widehat\beta_{\mathrm G}-\beta_\star\|_G^2\mid\Psi\right]
=
\lambda^2(\beta_\star-\beta_{\obs})^\top AGA(\beta_\star-\beta_{\obs})
+
\sigma^2\operatorname{tr}(GASA),
\]
while
\[
\E\!\left[\|\widehat\beta_{\expd}-\beta_\star\|_G^2\mid\Psi\right]
=
\lambda^2\beta_\star^\top AGA\,\beta_\star
+
\sigma^2\operatorname{tr}(GASA).
\]
Therefore the linear grounded estimator is better than \EXPOnly{} if and only if
\begin{equation}
(\beta_\star-\beta_{\obs})^\top AGA(\beta_\star-\beta_{\obs})
\le
\beta_\star^\top AGA\,\beta_\star.
\label{eq:centering_crossover}
\end{equation}
\end{corollary}

\begin{proof}
First,
\[
\widehat\beta_{\expd}-\beta_\star
=
A\Psi^\top(\Psi\beta_\star+\varepsilon)-\beta_\star
=
(AS-I)\beta_\star+A\Psi^\top\varepsilon.
\]
Since $AS-I=-\lambda A$,
\begin{equation}
\widehat\beta_{\expd}-\beta_\star
=
-\lambda A\beta_\star+A\Psi^\top\varepsilon.
\label{eq:app_exp_error}
\end{equation}

For grounded, note that
\[
\Psi\beta_{\obs}-Y
=
\Psi(\beta_{\obs}-\beta_\star)-\varepsilon
=
\Psi\delta_\star-\varepsilon.
\]
Hence
\[
\widehat\beta_{\mathrm G}-\beta_\star
=
\beta_{\obs}-\beta_\star-A\Psi^\top(\Psi\delta_\star-\varepsilon)
=
\delta_\star-AS\delta_\star+A\Psi^\top\varepsilon.
\]
Using $I-AS=\lambda A$ gives
\begin{equation}
\widehat\beta_{\mathrm G}-\beta_\star
=
\lambda A\delta_\star+A\Psi^\top\varepsilon.
\label{eq:app_grounded_error}
\end{equation}

Now take conditional $G$-risk.
From \eqref{eq:app_exp_error},
\[
\E\!\left[\|\widehat\beta_{\expd}-\beta_\star\|_G^2\mid\Psi\right]
=
\lambda^2\beta_\star^\top AGA\,\beta_\star
+
\E\!\left[(A\Psi^\top\varepsilon)^\top G(A\Psi^\top\varepsilon)\mid\Psi\right],
\]
because the cross term vanishes by $\E[\varepsilon\mid\Psi]=0$.
The noise term is
\[
\E\!\left[(A\Psi^\top\varepsilon)^\top G(A\Psi^\top\varepsilon)\mid\Psi\right]
=
\operatorname{tr}\!\left(GA\Psi^\top\,\E[\varepsilon\varepsilon^\top\mid\Psi]\,\Psi A\right)
=
\sigma^2\operatorname{tr}(GASA).
\]
This gives the EXP-only expression.

The grounded formula is identical, replacing the bias vector $-\lambda A\beta_\star$ by $\lambda A\delta_\star$ from \eqref{eq:app_grounded_error}.
Comparing the two expressions yields \eqref{eq:centering_crossover}.
\end{proof}

%% file: main.bbl
\begin{thebibliography}{39}
\providecommand{\natexlab}[1]{#1}
\providecommand{\url}[1]{\texttt{#1}}
\expandafter\ifx\csname urlstyle\endcsname\relax
  \providecommand{\doi}[1]{doi: #1}\else
  \providecommand{\doi}{doi: \begingroup \urlstyle{rm}\Url}\fi

\bibitem[Bang \& Robins(2005)Bang and Robins]{BangRobins2005}
Heejung Bang and James~M. Robins.
\newblock Doubly robust estimation in missing data and causal inference models.
\newblock \emph{Biometrics}, 61\penalty0 (4):\penalty0 962--973, 2005.
\newblock \doi{10.1111/j.1541-0420.2005.00377.x}.

\bibitem[Bareinboim \& Pearl(2016)Bareinboim and Pearl]{BareinboimPearl2016DataFusion}
Elias Bareinboim and Judea Pearl.
\newblock Causal inference and the data-fusion problem.
\newblock \emph{Proceedings of the National Academy of Sciences}, 113\penalty0 (27):\penalty0 7345--7352, 2016.
\newblock \doi{10.1073/pnas.1510507113}.
\newblock URL \url{https://www.pnas.org/doi/10.1073/pnas.1510507113}.

\bibitem[Cheng \& Cai(2021)Cheng and Cai]{Cheng2021Adaptive}
David Cheng and Tianxi Cai.
\newblock Adaptive combination of randomized and observational data, 2021.
\newblock URL \url{https://arxiv.org/abs/2111.15012}.

\bibitem[Chiang et~al.(2024)Chiang, Zheng, Sheng, Angelopoulos, Li, Li, Zhu, Zhang, Jordan, Gonzalez, and Stoica]{Chiang2024ChatbotArena}
Wei-Lin Chiang, Lianmin Zheng, Ying Sheng, Anastasios~Nikolas Angelopoulos, Tianle Li, Dacheng Li, Banghua Zhu, Hao Zhang, Michael Jordan, Joseph~E. Gonzalez, and Ion Stoica.
\newblock Chatbot arena: An open platform for evaluating {LLM}s by human preference.
\newblock In \emph{Proceedings of the 41st International Conference on Machine Learning}, volume 235 of \emph{Proceedings of Machine Learning Research}, pp.\  8359--8388, 2024.
\newblock URL \url{https://proceedings.mlr.press/v235/chiang24b.html}.

\bibitem[Christiano et~al.(2017)Christiano, Leike, Brown, Martic, Legg, and Amodei]{Christiano2017HumanPreferences}
Paul~F. Christiano, Jan Leike, Tom~B. Brown, Miljan Martic, Shane Legg, and Dario Amodei.
\newblock Deep reinforcement learning from human preferences.
\newblock In \emph{Advances in Neural Information Processing Systems}, volume~30, 2017.
\newblock URL \url{https://arxiv.org/abs/1706.03741}.

\bibitem[Cole \& Stuart(2010)Cole and Stuart]{ColeStuart2010}
Stephen~R. Cole and Elizabeth~A. Stuart.
\newblock Generalizing evidence from randomized clinical trials to target populations: The {ACTG 320} trial.
\newblock \emph{American Journal of Epidemiology}, 172\penalty0 (1):\penalty0 107--115, 2010.
\newblock \doi{10.1093/aje/kwq084}.

\bibitem[Colnet et~al.(2024)Colnet, Mayer, Chen, Dieng, Li, Varoquaux, Vert, Josse, and Yang]{Colnet2024Review}
B{\'e}n{\'e}dicte Colnet, Imke Mayer, Guanhua Chen, Awa Dieng, Ruohong Li, Ga{\"e}l Varoquaux, Jean-Philippe Vert, Julie Josse, and Shu Yang.
\newblock Causal inference methods for combining randomized trials and observational studies: A review.
\newblock \emph{Statistical Science}, 39\penalty0 (1):\penalty0 165--191, 2024.
\newblock \doi{10.1214/23-STS889}.
\newblock URL \url{https://doi.org/10.1214/23-STS889}.

\bibitem[Dubois et~al.(2024)Dubois, Galambosi, Liang, and Hashimoto]{Dubois2024LengthControlledAlpacaEval}
Yann Dubois, Bal{\'a}zs Galambosi, Percy Liang, and Tatsunori~B. Hashimoto.
\newblock Length-controlled alpacaeval: A simple way to debias automatic evaluators.
\newblock In \emph{First Conference on Language Modeling (COLM)}, 2024.
\newblock URL \url{https://openreview.net/forum?id=CybBmzWBX0}.

\bibitem[Dud{\'{i}}k et~al.(2011)Dud{\'{i}}k, Langford, and Li]{Dudik2011DoublyRobust}
Miroslav Dud{\'{i}}k, John Langford, and Lihong Li.
\newblock Doubly robust policy evaluation and learning.
\newblock In \emph{Proceedings of the 28th International Conference on Machine Learning}, 2011.
\newblock URL \url{https://arxiv.org/abs/1103.4601}.

\bibitem[Dud{\'{i}}k et~al.(2014)Dud{\'{i}}k, Erhan, Langford, and Li]{Dudik2014DRStatSci}
Miroslav Dud{\'{i}}k, Dumitru Erhan, John Langford, and Lihong Li.
\newblock Doubly robust policy evaluation and optimization.
\newblock \emph{Statistical Science}, 29\penalty0 (4):\penalty0 485--511, 2014.
\newblock \doi{10.1214/14-STS500}.
\newblock URL \url{https://doi.org/10.1214/14-STS500}.

\bibitem[Feder et~al.(2022)Feder, Keith, Manzoor, Pryzant, Sridhar, Wood-Doughty, Eisenstein, Grimmer, Reichart, Roberts, Stewart, Veitch, and Yang]{FederEtAl2022CausalNLP}
Amir Feder, Katherine~A. Keith, Emaad Manzoor, Reid Pryzant, Dhanya Sridhar, Zach Wood-Doughty, Jacob Eisenstein, Justin Grimmer, Roi Reichart, Margaret~E. Roberts, Brandon~M. Stewart, Victor Veitch, and Diyi Yang.
\newblock Causal inference in natural language processing: Estimation, prediction, interpretation and beyond.
\newblock \emph{Transactions of the Association for Computational Linguistics}, 10:\penalty0 1138--1158, 2022.
\newblock \doi{10.1162/tacl_a_00511}.

\bibitem[{Gemma Team}(2025)]{Gemma3TechnicalReport}
{Gemma Team}.
\newblock Gemma 3 technical report, 2025.
\newblock URL \url{https://arxiv.org/abs/2503.19786}.

\bibitem[Hermann et~al.(2015)Hermann, Ko{\v{c}}isk{\'y}, Grefenstette, Espeholt, Kay, Suleyman, and Blunsom]{Hermann2015CNN}
Karl~Moritz Hermann, Tom{\'a}{\v{s}} Ko{\v{c}}isk{\'y}, Edward Grefenstette, Lasse Espeholt, Will Kay, Mustafa Suleyman, and Phil Blunsom.
\newblock Teaching machines to read and comprehend.
\newblock In \emph{Advances in Neural Information Processing Systems}, volume~28, 2015.
\newblock URL \url{https://arxiv.org/abs/1506.03340}.

\bibitem[Jiang \& Li(2016)Jiang and Li]{Jiang2016DROffPolicy}
Nan Jiang and Lihong Li.
\newblock Doubly robust off-policy value evaluation for reinforcement learning.
\newblock In \emph{Proceedings of The 33rd International Conference on Machine Learning}, volume~48 of \emph{Proceedings of Machine Learning Research}, pp.\  652--661, 2016.
\newblock URL \url{https://proceedings.mlr.press/v48/jiang16.html}.

\bibitem[Jimenez et~al.(2024)Jimenez, Yang, Wettig, Yao, Pei, Press, and Narasimhan]{Jimenez2024SWEBench}
Carlos~E. Jimenez, John Yang, Alexander Wettig, Shunyu Yao, Kexin Pei, Ofir Press, and Karthik Narasimhan.
\newblock {SWE}-bench: Can language models resolve real-world {GitHub} issues?
\newblock In \emph{The Twelfth International Conference on Learning Representations}, 2024.
\newblock URL \url{https://arxiv.org/abs/2310.06770}.

\bibitem[Kallus et~al.(2018)Kallus, Puli, and Shalit]{Kallus2018ExperimentalGrounding}
Nathan Kallus, Aahlad~Manas Puli, and Uri Shalit.
\newblock Removing hidden confounding by experimental grounding.
\newblock In \emph{Advances in Neural Information Processing Systems}, volume~31, pp.\  10911--10920, 2018.
\newblock URL \url{https://proceedings.neurips.cc/paper/2018/hash/566f0ea4f6c2e947f36795c8f58ba901-Abstract.html}.

\bibitem[Liang et~al.(2023)Liang, Bommasani, Lee, Tsipras, Soylu, Yasunaga, Zhang, Narayanan, Wu, Kumar, Newman, Yuan, Yan, Zhang, Cosgrove, Manning, R{\'e}, Acosta-Navas, Hudson, Zelikman, Durmus, Ladhak, Rong, Ren, Yao, Wang, Santhanam, Orr, Zheng, Yuksekgonul, Suzgun, Kim, Guha, Chatterji, Khattab, Henderson, Huang, Chi, Xie, Santurkar, Ganguli, Hashimoto, Icard, Zhang, Chaudhary, Wang, Li, Mai, Zhang, and Koreeda]{Liang2023HELM}
Percy Liang, Rishi Bommasani, Tony Lee, Dimitris Tsipras, Dilara Soylu, Michihiro Yasunaga, Yian Zhang, Deepak Narayanan, Yuhuai Wu, Ananya Kumar, Benjamin Newman, Binhang Yuan, Bobby Yan, Ce~Zhang, Christian Cosgrove, Christopher~D. Manning, Christopher R{\'e}, Diana Acosta-Navas, Drew~A. Hudson, Eric Zelikman, Esin Durmus, Faisal Ladhak, Frieda Rong, Hongyu Ren, Huaxiu Yao, Jue Wang, Keshav Santhanam, Laurel Orr, Lucia Zheng, Mert Yuksekgonul, Mirac Suzgun, Nathan Kim, Neel Guha, Niladri Chatterji, Omar Khattab, Peter Henderson, Qian Huang, Ryan Chi, Sang~Michael Xie, Shibani Santurkar, Surya Ganguli, Tatsunori Hashimoto, Thomas Icard, Tianyi Zhang, Vishrav Chaudhary, William Wang, Xuechen Li, Yifan Mai, Yuhui Zhang, and Yuta Koreeda.
\newblock Holistic evaluation of language models.
\newblock \emph{Transactions on Machine Learning Research}, 2023.
\newblock URL \url{https://arxiv.org/abs/2211.09110}.

\bibitem[Lin et~al.(2025)Lin, Tarp, and Evans]{Lin2025PowerLikelihood}
Xi~Lin, Jens~Magelund Tarp, and Robin~J. Evans.
\newblock Combining experimental and observational data through a power likelihood.
\newblock \emph{Biometrics}, 81\penalty0 (1):\penalty0 ujaf008, 2025.
\newblock \doi{10.1093/biomtc/ujaf008}.
\newblock URL \url{https://academic.oup.com/biometrics/article/81/1/ujaf008/8016472}.

\bibitem[Mathews \& Nagappan(2025)Mathews and Nagappan]{Mathews2025BouncerBench}
Noble~Saji Mathews and Meiyappan Nagappan.
\newblock Is your automated software engineer trustworthy?, 2025.
\newblock URL \url{https://arxiv.org/abs/2506.17812}.

\bibitem[Nallapati et~al.(2016)Nallapati, Zhou, dos Santos, Gulcehre, and Xiang]{Nallapati2016Sum}
Ramesh Nallapati, Bowen Zhou, Cicero dos Santos, Caglar Gulcehre, and Bing Xiang.
\newblock Abstractive text summarization using sequence-to-sequence {RNN}s and beyond.
\newblock In \emph{Proceedings of The 20th SIGNLL Conference on Computational Natural Language Learning}, pp.\  280--290, 2016.
\newblock URL \url{https://arxiv.org/abs/1602.06023}.

\bibitem[{OpenAI}(2024)]{OpenAI2024SWEBenchVerified}
{OpenAI}.
\newblock Introducing {SWE}-bench {V}erified.
\newblock OpenAI Blog, 2024.
\newblock URL \url{https://openai.com/index/introducing-swe-bench-verified/}.

\bibitem[Ouyang et~al.(2022)Ouyang, Wu, Jiang, Almeida, Wainwright, Mishkin, Zhang, Agarwal, Slama, Ray, Schulman, Hilton, Kelton, Miller, Simens, Askell, Welinder, Christiano, Leike, and Lowe]{Ouyang2022InstructGPT}
Long Ouyang, Jeffrey Wu, Xu~Jiang, Diogo Almeida, Carroll~L. Wainwright, Pamela Mishkin, Chong Zhang, Sandhini Agarwal, Katarina Slama, Alex Ray, John Schulman, Jacob Hilton, Fraser Kelton, Luke Miller, Maddie Simens, Amanda Askell, Peter Welinder, Paul Christiano, Jan Leike, and Ryan Lowe.
\newblock Training language models to follow instructions with human feedback.
\newblock In \emph{Advances in Neural Information Processing Systems}, volume~35, 2022.
\newblock URL \url{https://proceedings.neurips.cc/paper_files/paper/2022/hash/b1efde53be364a73914f58805a001731-Abstract-Conference.html}.

\bibitem[Panickssery et~al.(2024)Panickssery, Bowman, and Feng]{PanicksseryBowmanFeng2024}
Arjun Panickssery, Samuel~R. Bowman, and Shi Feng.
\newblock {LLM} evaluators recognize and favor their own generations, 2024.
\newblock URL \url{https://arxiv.org/abs/2404.13076}.

\bibitem[Pearl \& Bareinboim(2011)Pearl and Bareinboim]{PearlBareinboim2011Transportability}
Judea Pearl and Elias Bareinboim.
\newblock Transportability of causal and statistical relations: A formal approach.
\newblock In \emph{Proceedings of the Twenty-Fifth AAAI Conference on Artificial Intelligence}, 2011.
\newblock \doi{10.1609/aaai.v25i1.7861}.

\bibitem[Rafailov et~al.(2023)Rafailov, Sharma, Mitchell, Manning, Ermon, and Finn]{Rafailov2023DPO}
Rafael Rafailov, Archit Sharma, Eric Mitchell, Christopher~D. Manning, Stefano Ermon, and Chelsea Finn.
\newblock Direct preference optimization: Your language model is secretly a reward model.
\newblock In \emph{Advances in Neural Information Processing Systems}, volume~36, 2023.
\newblock URL \url{https://proceedings.neurips.cc/paper_files/paper/2023/hash/a85b405ed65c6477a4fe8302b5e06ce7-Abstract-Conference.html}.

\bibitem[Robins et~al.(1994)Robins, Rotnitzky, and Zhao]{RobinsRotnitzkyZhao1994}
James~M. Robins, Andrea Rotnitzky, and Lue~Ping Zhao.
\newblock Estimation of regression coefficients when some regressors are not always observed.
\newblock \emph{Journal of the American Statistical Association}, 89\penalty0 (427):\penalty0 846--866, 1994.
\newblock \doi{10.1080/01621459.1994.10476818}.

\bibitem[Rosenman et~al.(2023)Rosenman, Basse, Owen, and Baiocchi]{Rosenman2023Shrinkage}
Evan T.~R. Rosenman, Guillaume Basse, Art~B. Owen, and Mike Baiocchi.
\newblock Combining observational and experimental datasets using shrinkage estimators.
\newblock \emph{Biometrics}, 79\penalty0 (4):\penalty0 2961--2973, 2023.
\newblock \doi{10.1111/biom.13827}.
\newblock URL \url{https://pubmed.ncbi.nlm.nih.gov/36629736/}.

\bibitem[See et~al.(2017)See, Liu, and Manning]{See2017Pointer}
Abigail See, Peter~J. Liu, and Christopher~D. Manning.
\newblock Get to the point: Summarization with pointer-generator networks.
\newblock In \emph{Proceedings of the 55th Annual Meeting of the Association for Computational Linguistics}, pp.\  1073--1083, 2017.
\newblock URL \url{https://arxiv.org/abs/1704.04368}.

\bibitem[Stiennon et~al.(2020)Stiennon, Ouyang, Wu, Ziegler, Lowe, Voss, Radford, Amodei, and Christiano]{Stiennon2020SummarizeHF}
Nisan Stiennon, Long Ouyang, Jeff Wu, Daniel~M. Ziegler, Ryan Lowe, Chelsea Voss, Alec Radford, Dario Amodei, and Paul Christiano.
\newblock Learning to summarize from human feedback.
\newblock In \emph{Advances in Neural Information Processing Systems}, volume~33, 2020.
\newblock URL \url{https://arxiv.org/abs/2009.01325}.

\bibitem[Swaminathan \& Joachims(2015)Swaminathan and Joachims]{Swaminathan2015CRM}
Adith Swaminathan and Thorsten Joachims.
\newblock Counterfactual risk minimization: Learning from logged bandit feedback.
\newblock In \emph{Proceedings of the 32nd International Conference on Machine Learning}, volume~37 of \emph{Proceedings of Machine Learning Research}, pp.\  814--823, 2015.
\newblock URL \url{https://proceedings.mlr.press/v37/swaminathan15.html}.

\bibitem[Thomas \& Brunskill(2016)Thomas and Brunskill]{Thomas2016DataEfficient}
Philip~S. Thomas and Emma Brunskill.
\newblock Data-efficient off-policy policy evaluation for reinforcement learning.
\newblock In \emph{Proceedings of The 33rd International Conference on Machine Learning}, volume~48 of \emph{Proceedings of Machine Learning Research}, pp.\  2139--2148, 2016.
\newblock URL \url{https://proceedings.mlr.press/v48/thomasa16.html}.

\bibitem[Veitch et~al.(2020)Veitch, Sridhar, and Blei]{VeitchSridharBlei2020}
Victor Veitch, Dhanya Sridhar, and David~M. Blei.
\newblock Adapting text embeddings for causal inference.
\newblock In \emph{Proceedings of the 36th Conference on Uncertainty in Artificial Intelligence (UAI)}, volume 124 of \emph{Proceedings of Machine Learning Research}, pp.\  919--928, 2020.
\newblock URL \url{https://proceedings.mlr.press/v124/veitch20a.html}.

\bibitem[Verga et~al.(2024)Verga, Hofst{\"a}tter, Althammer, Su, Piktus, Arkhangorodsky, Xu, White, and Lewis]{VergaEtAl2024Juries}
Pat Verga, Sebastian Hofst{\"a}tter, Sophia Althammer, Yixuan Su, Aleksandra Piktus, Arkady Arkhangorodsky, Minjie Xu, Naomi White, and Patrick Lewis.
\newblock Replacing judges with juries: Evaluating {LLM} generations with a panel of diverse models, 2024.
\newblock URL \url{https://arxiv.org/abs/2404.18796}.

\bibitem[Wang et~al.(2023)Wang, Li, Chen, Zhu, Lin, Cao, Liu, Liu, and Sui]{WangEtAl2023FairEvaluators}
Peiyi Wang, Lei Li, Liang Chen, Dawei Zhu, Binghuai Lin, Yunbo Cao, Qi~Liu, Tianyu Liu, and Zhifang Sui.
\newblock Large language models are not fair evaluators, 2023.
\newblock URL \url{https://arxiv.org/abs/2305.17926}.

\bibitem[Yang et~al.(2024)Yang, Jimenez, Wettig, Lieret, Yao, Narasimhan, and Press]{YangEtAl2024SWEAgent}
John Yang, Carlos~E. Jimenez, Alexander Wettig, Kilian Lieret, Shunyu Yao, Karthik~R. Narasimhan, and Ofir Press.
\newblock {SWE}-agent: Agent-computer interfaces enable automated software engineering.
\newblock In \emph{Advances in Neural Information Processing Systems 37 (NeurIPS)}, 2024.
\newblock URL \url{https://openreview.net/forum?id=mXpq6ut8J3}.

\bibitem[Yang et~al.(2025)Yang, Lin, Athey, Jordan, and Imbens]{Yang2025CVCI}
Xuelin Yang, Licong Lin, Susan Athey, Michael~I. Jordan, and Guido~W. Imbens.
\newblock Cross-validated causal inference: a modern method to combine experimental and observational data, 2025.
\newblock URL \url{https://arxiv.org/abs/2511.00727}.

\bibitem[Zhao et~al.(2024)Zhao, Ren, Hessel, Cardie, Choi, and Deng]{Zhao2024WildChat}
Wenting Zhao, Xiang Ren, Jack Hessel, Claire Cardie, Yejin Choi, and Yuntian Deng.
\newblock Wildchat: 1{M} {ChatGPT} interaction logs in the wild.
\newblock In \emph{The Twelfth International Conference on Learning Representations}, 2024.
\newblock URL \url{https://arxiv.org/abs/2405.01470}.

\bibitem[Zheng et~al.(2023)Zheng, Chiang, Sheng, Zhuang, Wu, Zhuang, Lin, Li, Li, Xing, Zhang, Gonzalez, and Stoica]{Zheng2023LLMJudge}
Lianmin Zheng, Wei-Lin Chiang, Ying Sheng, Siyuan Zhuang, Zhanghao Wu, Yonghao Zhuang, Zi~Lin, Zhuohan Li, Dacheng Li, Eric~P. Xing, Hao Zhang, Joseph~E. Gonzalez, and Ion Stoica.
\newblock Judging {LLM}-as-a-judge with {MT}-bench and chatbot arena.
\newblock In \emph{Advances in Neural Information Processing Systems}, volume~36, 2023.
\newblock URL \url{https://arxiv.org/abs/2306.05685}.

\bibitem[Zheng et~al.(2024)Zheng, Chiang, Sheng, Li, Zhuang, Wu, Zhuang, Li, Lin, Xing, Gonzalez, Stoica, and Zhang]{Zheng2023LMSYSChat1M}
Lianmin Zheng, Wei-Lin Chiang, Ying Sheng, Tianle Li, Siyuan Zhuang, Zhanghao Wu, Yonghao Zhuang, Zhuohan Li, Zi~Lin, Eric~P. Xing, Joseph~E. Gonzalez, Ion Stoica, and Hao Zhang.
\newblock {LMSYS}-{Chat}-1{M}: A large-scale real-world {LLM} conversation dataset.
\newblock In \emph{The Twelfth International Conference on Learning Representations}, 2024.
\newblock URL \url{https://openreview.net/forum?id=BOfDKxfwt0}.

\end{thebibliography}
